\documentclass[twoside]{article}
\usepackage{amsmath,amsfonts,bm}

\def\1{\bm{1}}

\def\vtheta{{\bm{\theta}}}

\def\vp{{\bm{p}}}

\def\vs{{\bm{s}}}

\def\vu{{\bm{u}}}
\def\vv{{\bm{v}}}

\def\vx{{\bm{x}}}
\def\vy{{\bm{y}}}

\def\evtheta{{\theta}}

\def\evp{{p}}

\def\evy{{y}}

\def\mA{{\bm{A}}}
\def\mB{{\bm{B}}}

\DeclareMathAlphabet{\mathsfit}{\encodingdefault}{\sfdefault}{m}{sl}
\SetMathAlphabet{\mathsfit}{bold}{\encodingdefault}{\sfdefault}{bx}{n}

\newcommand{\E}{\mathbb{E}}

\DeclareMathOperator*{\argmin}{arg\,min}

\usepackage[accepted]{aistats2021}
\usepackage[round]{natbib}

\usepackage[utf8]{inputenc} % allow utf-8 input
\usepackage[T1]{fontenc}    % use 8-bit T1 fonts
\usepackage{url}            % simple URL typesetting
\usepackage{booktabs}       % professional-quality tables
\usepackage{amsfonts}       % blackboard math symbols
\usepackage{nicefrac}       % compact symbols for 1/2, etc.
\usepackage{microtype}      % microtypography
\usepackage{setspace}
\usepackage{algorithm,algpseudocode}
\usepackage[algo2e,ruled,norelsize]{algorithm2e}
\usepackage{graphicx}
\usepackage{amsthm}
\usepackage{arydshln}
\usepackage{bbding}
\usepackage{multirow}
\usepackage{makecell}
\usepackage{enumitem}
\usepackage{overpic}
\usepackage{wrapfig}
\usepackage{xcolor, colortbl}
\usepackage{subfigure}
\usepackage[pagebackref=true,breaklinks=true,colorlinks,bookmarks=false,citecolor=blue,linkcolor=blue]{hyperref}

\newtheorem{thm}{Theorem}

\newtheorem{lemma}{Lemma}

\DeclareGraphicsExtensions{.pdf,.jpg,.png}

\def\ie{\textit{i.e.,~}}
\def\eg{\textit{e.g.,~}}

\def\wrt{\textit{w.r.t.~}}
\def\vs{\textit{v.s.~}}
\def\resp{\textit{resp.~}}
\def\sota{state-of-the-art~}

\usepackage{cleveref}
\crefname{equation}{Eq.}{Eq.}
\crefname{figure}{Figure}{Fig.}
\crefname{table}{Table}{Table~}
\crefname{section}{Section}{Section~}
\crefname{algorithm}{Algorithm}{Algorithm~}
\crefname{thm}{Theorem}{Theorem~}

\begin{document}


\runningauthor{Xin-Yu Zhang, Taihong Xiao, Haolin Jia, Ming-Ming Cheng, Ming-Hsuan Yang}

\twocolumn[

\aistatstitle{Semi-Supervised Learning with Meta-Gradient}

\aistatsauthor{ Xin-Yu Zhang$^*$ \And Taihong Xiao$^*$  \And  Haolin Jia }

\aistatsaddress{ TKLNDST, CS, Nankai University \And  University of California, Merced \And Tongji University }

\aistatsauthor{ Ming-Ming Cheng \And Ming-Hsuan Yang }

\aistatsaddress{ TKLNDST, CS, Nankai University \And  University of California, Merced }

]

\begin{abstract}
In this work, we propose a simple yet effective meta-learning algorithm in semi-supervised learning. We notice that most existing consistency-based approaches suffer from overfitting and limited model generalization ability, especially when training with only a small number of labeled data. To alleviate this issue, we propose a learn-to-generalize regularization term by utilizing the label information and optimize the problem in a meta-learning fashion. Specifically, we seek the pseudo labels of the unlabeled data so that the model can generalize well on the labeled data, which is formulated as a nested optimization problem. We address this problem using the meta-gradient that bridges between the pseudo label and the regularization term. In addition, we introduce a simple first-order approximation to avoid computing higher-order derivatives and provide theoretic convergence analysis. Extensive evaluations on the SVHN, CIFAR, and ImageNet datasets demonstrate that the proposed algorithm performs favorably against \sota methods.
%
% The training sources will be released upon acceptance.
\end{abstract}

\section{Introduction}
The rapid advances of deep neural networks can be in part 
attributed to the availability of large-scale datasets with extensive annotations, which require considerable human labor.
However, a typical real-world scenario 
is that only a small amount of data has the
corresponding annotations while the majority of training examples are unlabeled.
Numerous semi-supervised learning (SSL) methods have since been developed in which
the unlabeled data are exploited to facilitate generalizing the learned models.

Existing SSL algorithms include co-training \citep{de1994learning,blum1998combining},
label propagation \citep{szummer2002partially}, graph regularization \citep{blum2001learning},
and the consistency-enforcing approaches \citep{rasmus2015semi,sajjadi2016regularization,
Laine2017temporal,miyato2018virtual,tarvainen2017mean,athiwaratkun2018there,yu2019tangent}.
Notably, the consistency-based approaches treat SSL as a generalization problem
and enforce consistent predictions
against small perturbations of the input data or model parameters.
The basic assumption is that similar training examples are more likely to
belong to the same category, so that the predictions of the network
in multiple passes should be consistent \citep{sajjadi2016regularization}.
As such, the consistency-based approaches are essentially designing pseudo labels
from the predictions of the same input signals,
while the incorrect predictions may misguide the training process. 
To improve the quality of the pseudo labels, two orthogonal directions,
\ie dedicating to carefully designed perturbations \citep{miyato2018virtual,yu2019tangent}
and delving into better role models \citep{Laine2017temporal,tarvainen2017mean}, 
have been introduced. 
%
% Beyond this, Laine \etal \citep{Laine2017temporal} and Tarvainen \etal \citep{tarvainen2017mean}
% borrow the idea of ensemble learning and distill knowledge from the predictions
% learned in ensemble to the student network.
%
Aside from the aforementioned approaches, \citet{athiwaratkun2018there} analyze the training dynamics of the models trained with the consistency regularization
and propose a variant of the stochastic weight averaging (SWA)~\citep{izmailov2018averaging}, \ie fastSWA, to improve performance and accelerate convergence.

From the above, we can see that the label information in most SSL methods is commonly used for pseudo labeling or label propagation especially for the unlabeled data apart from providing the ground truth for the labeled data. 
Either pseudo labeling or label propagation is based on the assumption of inner structure of the data manifold. For example, the regularization terms in $\Pi$-Model~\citep{Laine2017temporal} and VAT~\citep{miyato2018virtual} are assuming that the decision boundary should be flat near the data. However, these consistency regularization terms are quite empirical and generic, where the given label information is not explicitly exploited in the consistency regularization. As a result, the model may overfit to the limited labeled data at a bad local minimum, which results in limited model generalization ability.

% Viewed from another perspective, the consistency-based algorithms
% can be broadly formulated as generating pseudo labels for the training
% examples, especially for the unlabeled data, in a way that
% the model can generalize well on the unseen examples.
%
% Previous work improves the quality of the pseudo labels by delving into either
% the representation perturbations \citep{miyato2018virtual} or ways of generating
% pseudo-labels in ensemble \citep{Laine2017temporal,tarvainen2017mean}.
%
% However, the label information, which is the only task-specific information for SSL,
% is not well exploited in the previous consistency regularization. It is common that the label is used for propagation of 
% % 
% For example, the regularization terms in $\Pi$-Model~[6] and VAT~[8] are based on the assumption that the decision boundary should be flat near the data.  However, these consistency regularization terms are unsupervised and generic, that is, not specially designed for the underlying task, which limits the generalization ability and performance of the learned model.

To alleviate this issue, we propose a meta-learning algorithm in which the pseudo labels are designed explicitly for generalization on the task of interest. Specifically, We regard the labeled data as a validation set, and generate pseudo-labels of the unlabeled data by minimizing the validation loss. Thereby, the label information in the validation loss can influence the pseudo-labels via the meta-gradients to influence the end-to-end training of the model parameters. In this way, the validation loss term, as the proposed new regularization term in this paper, explicitly includes the label information and drives the model towards better generalization ability, as indicated by the decrease of validation loss (see \cref{thm:convergence}).
We further introduce a simple first-order approximation to alleviate the issue of computing higher-order derivatives, and an improved training protocol to address the sample bias problem.
Under mild conditions, the proposed meta-learning algorithm enjoys a convergence rate of $O(1/\epsilon^2)$,
which is identical to that of the stochastic gradient descent (SGD) algorithm.
% The convergence of the exact meta-learning algorithm is guaranteed
% under mild conditions.
%
Extensive experimental results demonstrate that our method performs favorably against the \sota approaches on the SVHN~\citep{netzer2011reading}, CIFAR~\citep{krizhevsky2009learning}, and ImageNet~\citep{ILSVRC15} datasets, and the ablation studies validate the effectiveness of each component of our approach.

\section{Related Work}
% \vspace{-3pt}

\paragraph{Consistency-based Semi-Supervised Learning.}
The consistency regularization term measures the discrepancy between
the predictions and the pseudo labels,
which are typically generated by the same data
with small perturbations on the input signals \citep{sajjadi2016regularization,miyato2018virtual}
or model parameters \citep{tarvainen2017mean}.
For the $\Pi$-model \citep{sajjadi2016regularization}, the predictions and pseudo labels
are generated by the same model with different data augmentations through different forward passes\footnote{Due to the existence of randomized operations, such as dropout \citep{hinton2012improving}
and shake-shake regularization \citep{Gastaldi17ShakeShake},
the outputs of the same input signal may be different in multiple forward passes.}.
\citet{Laine2017temporal} propose the temporal ensembling approach to improve the quality of pseudo labels by keeping an exponential moving average (EMA) of the history predictions of each training example.
However, the scalability of this method is limited since the memory footprint grows linearly with the number of training examples.
Instead, \citet{tarvainen2017mean} present the mean teacher method to track the model parameters and generate pseudo labels using the teacher model parameterized by the EMA of the history model parameters.
On the other hand, \citet{miyato2018virtual} present the virtual adversarial training scheme to focus on disturbing the input data in an adversarial direction, and \citet{yu2019tangent} decouple the adversarial direction into
the tangent and normal directions of the embedded training data manifold.
With the dedicated perturbation directions, the robustness of the learned model can be significantly improved.
Aside from the above-mentioned methods, \citet{athiwaratkun2018there}
introduce the fastSWA method to average the model parameters along the training timeline.

Recently, there is a line of research focusing on mitigating the self-supervised algorithms to SSL,
in which heavy data augmentation is used to promote the performance \citep{berthelot2019mixmatch,berthelot2019remixmatch,sohn2020fixmatch}.
For example, MixMatch \citep{berthelot2019mixmatch} enforces consistent predictions between the differently-augmented training data,
and ReMixMatch \citep{berthelot2019remixmatch} further introduces an automatic augmentation strategy to effectively generate the strongly-augmented data,
for which the consistency term is enforced.
Since our approach adopts a simple \textit{mixup} augmentation, the direct comparison with these methods is unfair.
The integration of our method with more sophisticated data augmentation is left for future work.

We notice that existing SSL methods do not exploit the label information
when computing the consistency regularization, which leads to limited model generalization ability.
To alleviate this, we relate the consistency loss with the label information
by unfolding and differentiating through one optimization step.
In this way, the update of the pseudo labels is guided by the meta-gradients of the labeled data loss,
and the consistency loss is designed to improve the generalization ability
specially for the underlying task.
% or ongoing task?
% propose to take the task-specific information into consideration, and design the pseudo labels explicitly for generalization.
%
%The favorable performance over the \sota methods demonstrates the important role
%of labeled examples in the consistency regularization.
We also experimentally verify the important role of the label information
% that the labeled examples play an important role
in the effectiveness of consistency regularization.
% Moreover, we provide a theoretical convergence guarantee of the proposed algorithm
% under mild conditions, making our method more principled than the prior arts.

\paragraph{Optimization-based Meta-Learning.}
%The motivation of training on unlabeled examples with labeled data loss for SSL methods
%mainly comes from the optimization-based meta-learning algorithms
%This work is motivated by recent advances of
Numerous optimization-based meta-learning algorithms
\citep{pmlr-v70-finn17a,grant2018recasting,finn2018probabilistic,yoon2018bayesian,rajeswaran2019meta} have been developed in recent years. 
Notably, \citet{pmlr-v70-finn17a} formulate the  meta-learning problem in a nested optimization format,
where the inner loop imitates the process of adaptation,
while the outer loop focuses on optimizing the meta-objective.
The inner-optimization is further replaced by a single SGD step so that the meta-objective can be optimized in an end-to-end manner.
% To alleviate the heavy computational cost, MAML further replaces the inner-optimization
% with a single SGD step so that the meta-objective can be optimized in an end-to-end manner.
%
Thanks to its simplicity and effectiveness, optimization-based meta-learning algorithms
have been applied to a wide range of vision and learning problems
including example re-weighting \citep{ren18l2rw},
neural architecture search \citep{liu2018darts},
and unrolled generative models \citep{unrolled2017metz}.
In this work, we develop a meta-learning algorithm in the semi-supervised settings
%
% The meta-learning technique addresses the aforementioned challenge of relating
% the consistency loss with the label information.
and demonstrate the potential of meta-learning for these tasks.
In addition, we present the theoretical convergence analysis of the proposed algorithm.
%
% To be specific, we prove that under mild conditions, the meta-learning algorithm is guaranteed
% to converge and enjoys the same convergence rate as the regular SGD.
% our exact meta-learning algorithm is guaranteed to
% converge to a critical point of the meta-objective under mild conditions.
%
% We demonstrate that meta-learning is a fruitful research direction for SSL.
% Different from existing approaches, however, our algorithm brings only one additional hyper-parameter,
% \ie the meta learning rate, and is guaranteed to converge to a critical point of the meta-objective
% under mild conditions. % (See \cref{}.)

A concurrent work \citep{wang2020meta} proposes a similar meta-learning algorithm for SSL.
Different from ours, however, their method learns a weight for each unlabeled data,
and the pseudo-labels are inherited from previous consistency-based methods (\eg \citet{tarvainen2017mean}).
We instead infer the pseudo-labels within the meta-learning framework,
making our method relatively independent of the previous consistency-based counterparts.

% \vspace{-3pt}
\section{Proposed Algorithm}
% \vspace{-3pt}

In a typical semi-supervised setting, we are given a few labeled data
$\mathcal{D}^l = \{(\vx_{k}^{l}, \vy_{k}):k = 1,\cdots,N^{l}\}$ and
a large amount of unlabeled data
$\mathcal{D}^u = \{\vx_{i}^{u}:i=1,\cdots,N^{u}\}$,
where $N^{l} \ll N^{u}$.
The goal is to train a classifier that generalizes well on the unseen test data drawn from the same distribution.
In the following, we present the algorithmic details.
For presentation clarity, a table of notations is provided in the supplementary material. 
%MH: no need to do this
%We develop an exact meta-learning algorithm in %\cref{sec:learning-to-generalize}, introduce the first-order approximation in \cref{sec:first-order}, and formulate the full algorithm with the improved training protocol in \cref{sec:mix-up}.
%
%Finally, we analyze the convergence behavior of the exact meta-learning algorithm in \cref{sec:convergence}.
%
%For presentation clarity, a table of notations is provided in the supplementary materials.

\subsection{Learning to Generalize}
\label{sec:learning-to-generalize}
% The main challenge of SSL lies in how to sufficiently utilize the unlabeled data.
% %
% In this work, we develop a learning-to-generalize algorithm to assign proximal labels for the unlabeled examples so that the model can generalize well on the labeled data.
%
% Concretely, we treat the labeled data as a validation set and learn to assign labels
% of the unlabeled data so that the classifier trained on the assigned pseudo-labels has
% minimal loss on the validation set.
% 
Let $f(\vx; \vtheta)$ be a generic classifier parameterized by $\vtheta$ and $\Phi(\vp,\vy)$ be a non-negative function that measures the discrepancy between
distributions $\vp$ and $\vy$.
We further assume $\Phi(\vp,\vy) = 0$ if and only if $\vp = \vy$, and thus
if $\vy$ is fixed, then $\vp = \vy$ is the (global) minima of the function $\Phi(\cdot,\vy)$.
We formulate the loss of $(\vx, \vy)$ as
$\mathcal{L}(\vx, \vy; \vtheta) = \Phi(f(\vx; \vtheta),\vy)$.
The learning-to-generalize problem is then formulated as following:
\begin{equation}
    \begin{split}
        \min_{\mathcal{Y}}&~\sum_{k=1}^{N^{l}}
        \mathcal{L}(\vx_{k}^{l}, \vy_{k}; \vtheta^{*}(\mathcal{Y})) \\
        \mbox{s.t.}&~\vtheta^{*}(\mathcal{Y}) = \argmin_{\theta}
        \sum_{i=1}^{N^{u}} \mathcal{L}(\vx_{i}^{u}, \widehat{\vy}_{i}; \vtheta),
    \end{split}
    \label{eqn:orig-formulation}
\end{equation}
where $\mathcal{Y} = \{\widehat{\vy}_{i}:i=1,\cdots,N^{u}\}$ denotes
the pseudo labels of the unlabeled data.
Here we consider the training labeled data $\{\vx_k^l: k=1,\ldots, N^l\}$ as the validation set and seek the optimal pseudo labels $\mathcal{Y}$ by minimizing the validation loss on the labeled data. 
However, the validation loss also depends on the optimal model parameters that are obtained by minimizing the unlabeled data loss using pseudo labels $\mathcal{Y}$.
Thus the above formulation is a nest optimization problem.

Solving the nested minimization problem exactly is computationally prohibitive
because calculating the gradients of the outer loop requires an entire inner optimization.
Thus, we online approximate the outer loop gradients in a way similar to \citet{ren18l2rw}.
Specifically, we adapt the generated pseudo labels based on the current mini-batch and
replace the inner optimization with a single SGD step.
As such, the descent direction of the pseudo labels is guided by the back-propagated signals of the labeled data loss.

Consider the gradient-based deep learning framework in which gradients are calculated at the mini-batch level and the SGD-like optimizer is used to update model parameters.
With a little bit abuse of notations, at the $t^{th}$ training step, a mini-batch of labeled data $\{(\vx_{k}^{l}, \vy_{k}):k=1,\cdots,B^{l}\}$ and
a mini-batch of unlabeled data $\{\vx_{i}^{u}:i=1,\cdots,B^{u}\}$ are sampled,
where $B^{l}$ and $B^{u}$ denote the batch sizes of the labeled data and unlabeled
data, respectively.
The pseudo labels of the unlabeled data are initialized as the current predictions of the classifier:
\begin{equation}
    \widetilde{\vy}_{i} = f(\vx_{i}^{u}; \vtheta_t).
    \label{eqn:init}
\end{equation}
We then compute the unlabeled data loss and the gradient \wrt the model parameters:
\begin{equation}
    \begin{split}
        \mathcal{L}(\vx_{i}^{u}, \widetilde{\vy}_{i}; \vtheta_t) &=
        \Phi(f(\vx_{i}^{u}; \vtheta_t),\widetilde{\vy}_{i}), \\
        \nabla \vtheta_t &= \frac{1}{B^{u}} \sum_{i=1}^{B^{u}} \nabla_{\vtheta}
        \mathcal{L}(\vx_{i}^{u}, \widetilde{\vy}_{i}; \vtheta_t).
    \end{split}
    \label{eqn:unlabel-first}
\end{equation}
Note that since the initialized pseudo labels are precisely the predictions of the classifier,
the unlabeled loss achieves minimum value - zero. Thus the gradient is zero, \ie $\vtheta_t=0$.
However, the Jacobian matrix of $\nabla \vtheta_t$ \wrt the pseudo labels is not necessarily
a zero matrix, thus making optimization via differentiating $\nabla \vtheta_t$ possible.
  
We apply one SGD step on the model parameters:
\begin{equation}
    \widetilde{\vtheta}_{t+1} = \vtheta_t - \alpha_{t} \nabla \vtheta_t,
    \label{eqn:pseudo-update}
\end{equation}
where $\alpha_{t}$ is the learning rate of the inner loop.
The SGD step is then evaluated on the labeled data
and the labeled data loss is treated as the meta-objective.
We differentiate the meta-objective through the SGD step and compute the meta-gradient \wrt the pseudo labels:
\begin{equation}
    \begin{split}
        \mathcal{L}(\vx_{k}^{l}, \vy_{k}; \widetilde{\vtheta}_{t+1}) &=
        \Phi(f(\vx_{k}^{l}; \widetilde{\vtheta}_{t+1}),\vy_{k}), \\
        \nabla \widetilde{\vy}_{i} &= \frac{1}{B^{l}}
        \sum_{k=1}^{B^{l}} \nabla_{\widetilde{\vy}_{i}}
        \mathcal{L}(\vx_{k}^{l}, \vy_{k}; \widetilde{\vtheta}_{t+1}).
    \end{split}
    \label{eqn:label}
\end{equation}

Note that by unfolding one SGD step, the labeled data loss is
related to the pseudo labels of the unlabeled data.
Moreover, since the labeled data loss serves as the meta-objective to be differentiated,
the update of the pseudo labels is guided by the label information, \ie the meta-gradients,
and thus concerns the specific task on interest.
Similar techniques are developed in the optimization-based meta-learning literature
\citep{pmlr-v70-finn17a} and employed in a wide range of applications
\citep{ren18l2rw,liu2018darts,unrolled2017metz,liu2019generative}.

\begin{figure*}
    \centering
\begin{minipage}[!t]{0.46\textwidth}
    \null
    \begin{algorithm}[H]
        % \LinesNumbered
        \setstretch{1.25}
        \caption{Meta-Learning Algorithm.}
        \label{alg:exact}
        \KwIn{regular learning rates $\{ \alpha_{t} \}$, \\
              \hspace{27pt} meta learning rates $\{ \beta_{t} \}$}
        \For{$t := 1$ \textrm{to \#iters}}
        {
        $\{(\vx_{k}^{l}, \vy_{k})\}_{k=1}^{B^{l}} \leftarrow \text{BatchSampler}(\mathcal{D}^{l})$ \\
        % \algorithmiccomment{Sample a mini-batch of labeled data} \\
        %
        $\{\vx_{i}^{u}\}_{i=1}^{B^{u}} \leftarrow \text{BatchSampler}(\mathcal{D}^{u})$ \\
        % \algorithmiccomment{Sample a mini-batch of unlabeled data} \\
        %
        $\widetilde{\vy}_{i} = f(\vx_{i}^{u}; \vtheta_t)$ \\
        % \algorithmiccomment{Initialize the pseudo labels} \\
        %
        $\mathcal{L}(\vx_{i}^{u}, \widetilde{\vy}_{i}; \vtheta_t) = \Phi(f(\vx_{i}^{u}; \vtheta_t),\widetilde{\vy}_{i})$ \\
        % \algorithmiccomment{Compute unlabeled loss ($\mathcal{L}(\vx_{i}^{u}, \widetilde{\vy}_{i}; \vtheta_t) = 0$)} \\
        %
        $\nabla \vtheta_t = \frac{1}{B^{u}} \sum_{i=1}^{B^{u}} \nabla_{\vtheta} \mathcal{L}(\vx_{i}^{u}, \widetilde{\vy}_{i}; \vtheta_t)$ \\
        % \algorithmiccomment{Gradient back-propagation ($\nabla \vtheta_t = 0$)} \\
        % 
        $\widetilde{\vtheta}_{t+1} = \vtheta_t - \alpha_{t} \nabla \vtheta_t$ \\
        % \algorithmiccomment{One SGD step on model parameters ($\widetilde{\vtheta}_{t+1} = \vtheta_t$)} \\
        % 
        $\mathcal{L}(\vx_{k}^{l}, \vy_{k}; \widetilde{\vtheta}_{t+1}) = \Phi(f(\vx_{k}^{l}; \widetilde{\vtheta}_{t+1}),\vy_{k})$ \\
        % \algorithmiccomment{Validate on the labeled data} \\
        % 
        $\nabla \widetilde{\vy}_{i} = \frac{1}{B^{l}} \sum_{k=1}^{B^{l}} \nabla_{\widetilde{\vy}_{i}} \mathcal{L}(\vx_{k}^{l}, \vy_{k}; \widetilde{\vtheta}_{t+1})$ \\
        % \algorithmiccomment{Compute gradients of pseudo labels} \\
        % 
        $\widehat{\vy}_i = \widetilde{\vy}_i - \beta_{t} \nabla \widetilde{\vy}_{i}$ \\
        % \algorithmiccomment{One SGD step on pseudo labels}\label{step:start} \\
        % 
        $\mathcal{L}(\vx_{i}^{u}, \widehat{\vy}_i; \vtheta_t) = \Phi(f(\vx_{i}^{u}; \vtheta_t),\widehat{\vy}_i)$ \\
        % \algorithmiccomment{Compute unlabeled loss with updated pseudo labels} \\
        % 
        $\nabla \widehat{\vtheta}_t = \frac{1}{B^{u}} \sum_{i=1}^{B^{u}} \nabla_{\vtheta} \mathcal{L}(\vx_{i}^{u}, \widehat{\vy}_{i}; \vtheta_t)$ \\
        % \algorithmiccomment{Gradient back-propagation} \\
        % 
        $\vtheta_{t+1} = \mbox{Optimizer}(\vtheta_t, \nabla \widehat{\vtheta}_t, \alpha_{t})$
        % \algorithmiccomment{Update model parameters}\label{step:end}
        }
    \end{algorithm}
\end{minipage}
\hfill
\resizebox{.50\textwidth}{!}{%%
\begin{minipage}[H]{0.58\textwidth}
    \null
    \begin{algorithm}[H]
        \setstretch{1.25}
        \caption{Algorithm with Mix-Up Augmentation.}
        \label{alg:first-order}
        \KwIn{regular learning rates $\{ \alpha_{t} \}$, \\
        \hspace{27pt} meta learning rates $\{ \beta_{t} \}$}
        \For{$t := 1$ \textrm{to \#iters}}
        {
            $\{(\vx_{k}^{l}, \vy_{k})\}_{k=1}^{B} \leftarrow
            \text{BatchSampler}(\mathcal{D}^{l})$ \\
            % \algorithmiccomment{Sample a mini-batch of labeled data} \\
            %
            $\{\vx_{i}^{u}\}_{i=1}^{B} \leftarrow \text{BatchSampler}(\mathcal{D}^{u})$ \\
            % \algorithmiccomment{Sample a mini-batch of unlabeled data} \\
            %
            $\widetilde{\vy}_{i} = f(\vx_{i}^{u}; \vtheta_t)$ \\
            % \algorithmiccomment{Initialize the pseudo labels} \\
            %
            $\mathcal{L}^{\text{KL}}(\vx_{k}^{l}, \vy_{k}; \vtheta_t) =
            \Phi^{\text{KL}}(f(\vx_{k}^{l}; \vtheta_t),\vy_{k})$ \\
            % \algorithmiccomment{Compute labeled loss} \\
            %
            $\nabla \vtheta_t^l = \frac{1}{B} \sum_{k=1}^{B} \nabla_{\vtheta}
            \mathcal{L}^{\text{KL}}(\vx_{k}^{l}, \vy_{k}; \widetilde{\vtheta}_{t+1})$ \\
            % \algorithmiccomment{Compute gradients of labeled loss} \\
            %
            $\epsilon = 0.01 \|\nabla \vtheta_t^l\|^{-1}$ \\
            % \algorithmiccomment{Compute the step size for first-order approximation} \\
            %
            $\nabla \widetilde{\vy}_{i} = \epsilon^{-1}
            \left(f(\vx_{i}^{u}; \vtheta_t+\epsilon\nabla \vtheta_t^l) -
            f(\vx_{i}^{u}; \vtheta_t-\epsilon\nabla \vtheta_t^l)\right)$ \\
            % \algorithmiccomment{Compute gradients of pseudo labels} \\
            %
            $\widehat{\vy}_i = \widetilde{\vy}_i - \beta_{t} \nabla \widetilde{\vy}_{i}$ \\
            % \algorithmiccomment{One SGD step on pseudo labels} \\
            %
            $\lambda_{i} \leftarrow \mbox{Beta}(\gamma, \gamma)$ \\
            % \algorithmiccomment{Sample from Beta distribution} \\
            %
            $\vx_{i}^{\text{in}} = \lambda_i \vx_{i}^{l} + (1 - \lambda_i) \vx_{i}^{u}$ \\
            $\vy_{i}^{\text{in}} = \lambda_i \vy_{i} + (1 - \lambda_i) \widehat{\vy}_{i}$ \\
            % \algorithmiccomment{Adopt Mixup augmentation} \\
            $\mathcal{L}^{\text{KL}}_{\text{cls}}(\vx_{i}^{\text{in}}, \vy_{i}^{\text{in}};
            \vtheta_t) = \Phi^{\text{KL}}(f(\vx_{i}^{\text{in}}; 
            \vtheta_t),\vy_{i}^{\text{in}})$ \\
            % \algorithmiccomment{Compute interpolated data loss} \\
            %
            $\mathcal{L}^{\text{MSE}}_{\text{cons}}(\vx_{i}^{u}, \widehat{\vy}_i; \vtheta_t) =
            \Phi^{\text{MSE}}(f(\vx_{i}^{u}; \vtheta_t),\widehat{\vy}_i)$ \\
            % \algorithmiccomment{Compute consistency regularization} \\
            %
            $\nabla \widehat{\vtheta}_t = \frac{1}{B} \sum_{i=1}^{B}
            \left( \nabla_{\vtheta} \mathcal{L}^{\text{KL}}_{\text{cls}}
            + \nabla_{\vtheta} \mathcal{L}^{\text{MSE}}_{\text{cons}} \right)$ \\
            % \algorithmiccomment{Back-propagation} \\
            %
            $\vtheta_{t+1} = \mbox{Optimizer}(\vtheta_t, \nabla \widehat{\vtheta}_t, \alpha_{t})$
            % \algorithmiccomment{Update model parameters}
        }
    \end{algorithm}
\end{minipage}%
}

\end{figure*}

Finally, we perform one SGD step on the pseudo labels, 
\begin{equation}
    \widehat{\vy}_i = \widetilde{\vy}_i - \beta_{t} \nabla \widetilde{\vy}_{i},
   \label{eqn:update-pseudo labels}
\end{equation}
where $\beta_{t}$ is the meta learning rate,
and compute the consistency loss from the unlabeled data and the updated pseudo labels.
The meta-learning algorithm is summarized in \cref{alg:exact}.

\subsection{First-Order Approximation}
\label{sec:first-order}
In \cref{sec:learning-to-generalize}, the most computationally expensive operation is
differentiation through the SGD step in \cref{eqn:label},
as the second-order derivative is involved.
To avoid this, we apply the chain rule to the second-order derivative by substituting \cref{eqn:unlabel-first} and \cref{eqn:pseudo-update} into \cref{eqn:label}:
\begin{equation}\label{eqn:chain-rule}
\begin{split}
    &\frac{\partial \mathcal{L}}{\partial \widetilde{\evy}_{i,j}}(\vx_{k}^{l}, \vy_{k}; \widetilde{\vtheta}_{t+1}) \\
    &\phantom{L}= - \frac{\alpha_{t}}{B^{u}} \nabla_{\vtheta}^{\top} \frac{\partial \mathcal{L}}{\partial\widetilde{\evy}_{i,j}}
    (\vx_{i}^{u}, \widetilde{\vy}_{i}; \vtheta_t) \cdot \nabla_{\vtheta} \mathcal{L}(\vx_{k}^{l}, \vy_{k}; \vtheta_t).
\end{split}
\end{equation}
The gradient of the validation loss \wrt the pseudo labels can thus be formulated as
\begin{equation}\label{eqn:label-grad}
\resizebox{0.48\textwidth}{!}{$
\begin{aligned}
    \nabla \widetilde{\evy}_{i,j} =& \frac{1}{B^{l}} \sum_{k=1}^{B^{l}}
    \frac{\partial \mathcal{L}}{\partial \widetilde{\evy}_{i,j}} 
    (\vx_{k}^{l}, \vy_{k}; \vtheta_t) \\
    =& - \frac{\alpha_{t}}{B^{u}} \nabla_{\vtheta}^{\top}
    \frac{\partial \mathcal{L}}{\partial \widetilde{\evy}_{i,j}}
    (\vx_{i}^{u}, \widetilde{\vy}_{i}; \vtheta_t) \cdot 
    \left( \frac{1}{B^{l}} \sum_{k=1}^{B^{l}} \nabla_{\vtheta} \mathcal{L}
    (\vx_{k}^{l}, \vy_{k}; \vtheta_t) \right).
\end{aligned}
$}
\end{equation}
Let
\begin{equation}
    \nabla \vtheta_t^l = \frac{1}{B^{l}} \sum_{i=1}^{B^{l}} \nabla_{\vtheta} \mathcal{L}
    (\vx_{i}^{l}, \vy_{i}; \vtheta_t),
    \label{eqn:normal-grad}
\end{equation}
and then it can be easily shown with Taylor expansion that as $\epsilon \to 0$,
\begin{equation}\label{eqn:taylor}
\resizebox{0.48\textwidth}{!}{$
    \nabla \widetilde{\evy}_{i,j} = - \frac{\alpha_{t}}{2B^{u}\epsilon}
    \left( \frac{\partial \mathcal{L}}{\partial \widetilde{\evy}_{i,j}} 
    (\vx_{i}^{u}, \widetilde{\vy}_{i}; \vtheta_t + \epsilon \nabla \vtheta_t^l) 
    -\frac{\partial \mathcal{L}}{\partial \widetilde{\evy}_{i,j}} 
    (\vx_{i}^{u}, \widetilde{\vy}_{i}; \vtheta_t - \epsilon \nabla \vtheta_t^l) \right).
$}
\end{equation}
Thus, we adopt the first-order approximation and use a sufficiently small $\epsilon$ to approximate $\nabla \widetilde{\evy}_{i,j}$.
As suggested in \citet{liu2018darts}, we use $\epsilon = 0.01 / \| \nabla \vtheta_t^l \|_2$
in this work.

Furthermore, the gradients \wrt the pseudo labels can be calculated in the closed form.
Here, following the common practice of consistency-based SSL~\citep{tarvainen2017mean}, we adopt the Kullback--Leibler divergence loss
$\Phi^{\text{KL}}(\vp, \vy) = \sum_{n} \evy_n \log (\evy_n / \evp_n)$
as the regular labeled data loss, and the mean squared error (MSE) loss
$\Phi^{\text{MSE}}(\vp, \vy) = \|\vp - \vy\|_2^2$ for the consistency loss.
% Here, we consider two metrics of distribution discrepancy, namely, the
% Kullback--Leibler divergence $\Phi^{\text{KL}}(\vp, \vy) = \sum_{n} \evy_n \log{\frac{\evy_n}{\evp_n}}$
% and the Mean Squared Error (MSE) $\Phi^{\text{MSE}}(\vp, \vy) = \|\vp - \vy\|_2^2$.
%
For the MSE loss, the gradients \wrt the pseudo labels are approximated by:
\begin{equation}
    \nabla \widetilde{\vy}_{i} \approx \frac{\alpha_{t}}{B^{u}\epsilon} \left( f(\vx_{i}^{u};\vtheta_t +
    \epsilon \nabla \vtheta_t^l) - f(\vx_{i}^{u};\vtheta_t - \epsilon \nabla \vtheta_t^l) \right).
    \label{eqn:pseudo-update-elementwise}
\end{equation}

\subsection{Improved Training Protocol}
\label{sec:mix-up}
    
The above-discussed meta-learning algorithm utilizes the unlabeled examples to improve the generalization ability.
However, there still remains the sampling bias issue in SSL.
Motivated by the success of the \textit{mixup} augmentation \citep{zhang2018mixup} in SSL \citep{wang2019semi,berthelot2019mixmatch},
we incorporate the cross-domain mixup augmentation in the proposed meta-learning algorithm.
We first assume that the mini-batches of labeled and unlabeled data are of the same batch size, \ie $B^{l} = B^{u} = B$, and then interpolate between each pair of labeled and unlabeled examples to generate new training data.
Note that when generating the corresponding labels, we interpolate between
the actual labels $\vy_{i}$ of the labeled examples and the updated pseudo labels $\widehat{\vy}_{i}$ of the unlabeled examples, 
\begin{equation}\label{eqn:mix-up}
    \begin{aligned}
        \vx_{i}^{\text{in}} &= \lambda_{i} \vx_{i}^{l} + (1 - \lambda_{i}) \vx_{i}^{u}, \\
        \vy_{i}^{\text{in}} &= \lambda_{i} \vy_{i} + (1 - \lambda_{i}) \widehat{\vy}_{i},
    \end{aligned}
    \quad
    i = 1, \cdots, B,
\end{equation}
where $\lambda_{1}, \cdots, \lambda_{B}$ are \textit{i.i.d.} samples drawn from the
$\mbox{Beta}(\gamma, \gamma)$ distribution.
Finally, the total loss is formulated as
\begin{equation}
    \mathcal{L} = \underbrace{ \sum_{i=1}^{B} \mathcal{L}^{\text{KL}}(\vx_{i}^{\text{in}},\vy_{i}^{\text{in}};\vtheta_{t}) }_{\text{classification loss}} +
    \underbrace{ \sum_{i=1}^{B} \mathcal{L}^{\text{MSE}}(\vx_{i}^{u},\widehat{\vy}_{i};\vtheta_{t}) }_{\text{consistency loss}},
    \label{eqn:total-loss}
\end{equation}
and the algorithm with first-order approximation and mixup augmentation
is illustrated in \cref{alg:first-order}.

\subsection{Convergence Analysis}
\label{sec:convergence}
In this section, we present the convergence analysis of \cref{alg:exact}.
Due to of the scarcity of labeled examples, we assume all labeled data are sampled at each step, \ie $B^{l} = N^{l}$, and that the MSE loss is used in the unlabeled consistency loss (\cref{eqn:unlabel-first}).
Under mild conditions, we show that \cref{alg:exact} is guaranteed to converge to a critical point of the meta-objective (\cref{thm:convergence}),
and enjoys a convergence rate of $O(1/\epsilon^2)$
(\cref{thm:convergence-rate}), which is the same as the regular SGD.
The proofs are presented in the supplementary material.

\begin{thm}
    Let
    \begin{equation}
      G(\vtheta;\mathcal{D}^{l}) = \frac{1}{N^{l}} \sum_{k=1}^{N^{l}} \mathcal{L}(\vx_{k}^{l}, \vy_{k}; \vtheta_t)
    \end{equation}
    be the loss function of the labeled examples.
    Assume 
    \begin{enumerate}[label=(\roman*),itemindent=0pt]
        \item the gradient function $\nabla_{\vtheta}G$ is Lipschitz-continuous with a Lipschitz constant $L_0$;
        and
        \item the norm of the Jacobian matrix of $f$ \wrt $\vtheta$ is upper-bounded by a constant $M$, \ie
        \begin{equation}
          \left\| J_{\vtheta} f(\vx_{i}^{u}; \vtheta) \right\| \leq M,
          \quad
          \forall \, i \in \left\{1, \cdots, N^{u}\right\}. \hspace{30pt}
          \label{eqn:jacobi-norm}
        \end{equation}
      \end{enumerate}
    If the regular learning rate $\alpha_{t}$ and meta learning rate $\beta_{t}$
    satisfy $\alpha_{t}^{2} \beta_{t} < (4 M^{2} L_{0})^{-1}$,
    then each SGD step of \cref{alg:exact} will decrease the validation loss
    $G(\vtheta)$, regardless of the selected unlabeled examples, \ie
    \begin{equation}
        G(\vtheta_{t+1}) \leq G(\vtheta_{t}),
        \quad
        \mbox{for each}~t.
    \end{equation}
    Furthermore, the equality holds if and only if $\nabla \widetilde{\vy} = \bm{0}$
    for the selected unlabeled batch at the $t^{th}$ step.
    \label{thm:convergence}
\end{thm}

\begin{thm}
    Assume the same conditions as in \cref{thm:convergence}, and
    \begin{equation}
        \inf_{t} \left( \beta_{t} - 4 \alpha_{t}^{2}\beta_{t}^{2} M^{2} L_{0} \right) = D_1 > 0,
        \quad
        \inf_{t} \alpha_{t} = D_{2} > 0.
        \label{eqn:lr-condition}
    \end{equation}
    We further assume that the unlabeled dataset contains the labeled dataset,
    \ie $\mathcal{D}^{l} \subseteq \mathcal{D}^{u}$.
    Then, \cref{alg:exact} achieves $\E \left[ \| \nabla_{\vtheta}G(\vtheta_{t}) \|^{2} \right] \leq \epsilon$
    in $O(1/\epsilon^2)$ steps, \ie
    \begin{equation}
        \min_{1 \leq t \leq T} \E \left[ \| \nabla_{\vtheta}G(\vtheta_{t}) \|^{2} \right] \leq \frac{C}{\sqrt{T}},
    \end{equation}
    where $C$ is a constant independent of the training process.
    \label{thm:convergence-rate}
\end{thm}

\paragraph{Remarks.}
(i) The assumption in \eqref{eqn:jacobi-norm} is realistic.
Here, we assume the neural network $f$ is continuously differentiable \wrt $\vtheta$.
Due to the existence of norm-based regularization, \ie weight decay,
we can assume $\vtheta$ is optimized within a compact set in the parameter space.
The Jacobian function $J_{\vtheta}f$ is thus bounded within the compact set due to its continuity.
Furthermore, since there are finite training examples, the bound in \eqref{eqn:jacobi-norm} is plausible.
(ii) The conditions in \eqref{eqn:lr-condition} specify that the learning rates
$\alpha_{t}$ and $\beta_{t}$ can neither grow too large nor decay to zero too rapidly.
The step learning rate annealing strategy can satisfy this condition
as long as the initial learning rate is sufficiently small.
(iii) The condition $\mathcal{D}^{l} \subseteq \mathcal{D}^{u}$ can be satisfied
by incorporating the labeled data into the unlabeled set.

\begin{table*}[!t]
    \definecolor{mygray}{RGB}{220,220,220}
    \renewcommand{\arraystretch}{1.0}
    \centering
    %MH: describe whether the results are accuracy or error rates
    %\caption{Semi-supervised classification results on the SVHN, CIFAR-10, and CIFAR-100 datasets.}
    \caption{Semi-supervised classification error rates of the Conv-Large \citep{tarvainen2017mean}
    architecture on the SVHN, CIFAR-10, and CIFAR-100 datasets.
    The numbers of labeled data are 1k, 4k, and 10k for these three datasets, respectively.}
    \vspace{3pt}
    \begin{tabular}{l|ccc}
        \Xhline{1pt}
        Method & \textbf{SVHN} & \textbf{CIFAR-10} & \textbf{CIFAR-100} \\
        \hline
        $\Pi$-Model \citep{Laine2017temporal} & 4.82\% & 12.36\% & 39.19\% \\
        \rowcolor{mygray} TE \citep{Laine2017temporal} &  4.42\% &  12.16\% & 38.65\% \\
        MT \citep{tarvainen2017mean} & 3.95\% & 12.31\% & - \\
        \rowcolor{mygray} MT+SNTG \citep{luo2018smooth} & 3.86\% & 10.93\% & - \\
        VAT \citep{miyato2018virtual} & 5.42\% & 11.36\% & - \\
        \rowcolor{mygray} VAT+Ent \citep{miyato2018virtual} & 3.86\% & 10.55\% & - \\
        VAT+Ent+SNTG \citep{luo2018smooth} & 3.83\% & ~~9.89\% & - \\
        \rowcolor{mygray} VAT+VAdD \citep{park2018adversarial} & 3.55\% & ~~9.22\% & - \\
        MA-DNN \citep{chen2018semi} & 4.21\% & 11.91\% & 34.51\% \\
        % SaaS \citep{cicek2018saas} & 4.77\% & 13.22\% & - \\
        \rowcolor{mygray} Co-training \citep{qiao2018deep} & 3.29\% & ~~8.35\% &  34.63\% \\
        MT+fastSWA \citep{athiwaratkun2018there} & - & ~~9.05\% & 33.62\% \\
        \rowcolor{mygray} TNAR-VAE \citep{yu2019tangent} & 3.74\% & ~~8.85\% & - \\
        ADA-Net \citep{wang2019semi} & 4.62\% & 10.30\% & - \\
        \rowcolor{mygray} ADA-Net+fastSWA \citep{wang2019semi} & - &  ~~8.72\% & - \\
        DualStudent \citep{ke2019dual} & - & ~~8.89\% & 32.77\% \\
        \rowcolor{mygray} Ours &  \textbf{3.15\%} &  \textbf{~~7.78\%} & \textbf{30.74\%} \\
        \cdashline{1-4}[2pt/3pt]
        Fully-Supervised & 2.67\% & ~~4.88\% & 22.10\% \\
        \Xhline{1pt}
    \end{tabular}
    \label{tab:main-results}
\end{table*}

\begin{table*}[!t]
\begin{minipage}{0.58\textwidth}
\centering
\caption{
    Semi-supervised classification error rates of the 26-layer ResNet \citep{he2016deep}
    architecture with the shake-shake regularization \citep{Gastaldi17ShakeShake} on the CIFAR-10 and CIFAR-100 datasets.
    % \xth{Is it possible to include more methods for comparison? If not, explain the reason in the text.}
    %Xinyu: because fastSWA is the sota method with the ResNet backbone.
}
\label{tab:shake-shake}
\vspace{3pt}
\scalebox{0.9}{
\begin{tabular}{l|ccc|ccc}
    \Xhline{1pt}
    Dataset & \multicolumn{3}{c|}{\textbf{CIFAR-10}} & \multicolumn{3}{c}{\textbf{CIFAR-100}} \\ \cline{1-4}\cline{5-7}
    \#Images & 50k & 50k & 50k & 50k & 50k & 50k \\
    \#Labels & 1k & 2k & 4k & 6k & 8k & 10k \\\hline
    fastSWA & 6.6\% & 5.7\% & 5.0\% & - & - & 28.0\% \\
    Ours & \textbf{6.3\%} & \textbf{5.2\%} & \textbf{4.1\%} & \textbf{26.7\%} & \textbf{25.1\%} & \textbf{22.9\%} \\
    \Xhline{1pt}
\end{tabular}
}
\end{minipage}
\hfill
\begin{minipage}{0.38\textwidth}
\centering
\caption{
    Semi-supervised classification error rates on the ImageNet \citep{ILSVRC15} dataset.
    10\% training images are used as the labeled data.
}
\vspace{3pt}
\begin{tabular}{l|cc}
    \Xhline{1pt}
    Method & \textbf{Top-1} & \textbf{Top-5} \\ \hline
    Labeled-Only & 53.65\% & 31.01\% \\
    MT  & 49.07\% & 23.59\% \\
    Co-training  & 46.50\% & 22.73\% \\
    ADA-Net & 44.91\% & 21.18\% \\
    Ours & \textbf{44.87\%} & \textbf{18.88\%} \\\cdashline{1-3}[2pt/3pt]
    Fully-Supervised & 29.15\% & 10.12\% \\
    \Xhline{1pt}
\end{tabular}
\label{tab:imagenet}
\end{minipage}
\vspace{-8pt}
\end{table*}

\section{Experiments}

We evaluate the proposed algorithm on the SVHN \citep{netzer2011reading},
CIFAR \citep{krizhevsky2009learning}, and ImageNet \citep{ILSVRC15} datasets.
The 13-layer Conv-Large \citep{tarvainen2017mean} and
26-layer ResNet \citep{he2016deep} with the shake-shake regularization \citep{Gastaldi17ShakeShake}
are used as the backbone models.
%
%For more information of the datasets and the training details, please refer to the supplementary materials.
%
More implementation details can be found in the supplementary material.
The training sources are available at \url{https://github.com/Sakura03/SemiMeta}.
%MH: add this line
% The source code and trained models will be made available to the public. 

\subsection{Results on the SVHN and CIFAR}
%MH: describe whether the results are accuracy or error rates
%In \cref{tab:main-results}, we report the semi-supervised classification results of the proposed algorithm on the SVHN, CIFAR-10, and CIFAR-100 datasets with evaluations against the state-of-the-art methods.
In \cref{tab:main-results}, we report the semi-supervised classification error rates of the proposed algorithm and state-of-the-art methods on the SVHN, CIFAR-10, and CIFAR-100 datasets. 
The proposed meta-learning algorithm performs
favorably against the previous approaches on all three datasets.
%
% In addition, we explore different backbone architectures and
% compare with the \sota method \citep{athiwaratkun2018there}
% on the 26-layer ResNet \citep{he2016deep} model with the shake-shake regularization \citep{Gastaldi17ShakeShake}.
In addition, we explore the effectiveness of the proposed algorithm
on different backbone architectures and evaluate on the 26-layer ResNet \citep{he2016deep}
with the shake-shake regularization \citep{Gastaldi17ShakeShake}.
Since only a few previous papers include experiments on this backbone,
we just compare the performance with the ``fastSWA'' method \citep{athiwaratkun2018there}
which gives quite complete results and achieves the \sota accuracy.
\cref{tab:shake-shake} shows that the proposed algorithm performs favorably
under all different experimental configurations, even with fewer labeled examples,
indicating the efficacy of the consistency loss guided by the meta-gradients.
% the label-guided consistency loss.

% Recent work such as MixMatch~\citep{berthelot2019mixmatch} and ReMixMatch~\citep{berthelot2019remixmatch} uses much data augmentation to promote the performance, which is orthogonal to our work. Therefore, it is unfair to directly compare their results with ours. For fair comparison, we implement the proposed algorithm on top of the MixMatch baseline. The error rates on CIFAR-10 are $6.24\%$ for MixMatch (see Table 5 of MixMatch~\citep{berthelot2019mixmatch}) and $6.02\%$ for MixMatch+Meta with WRN-28 backbone and 4k labels.

\begin{figure*}[!t]
    \centering
    \includegraphics[width=\textwidth]{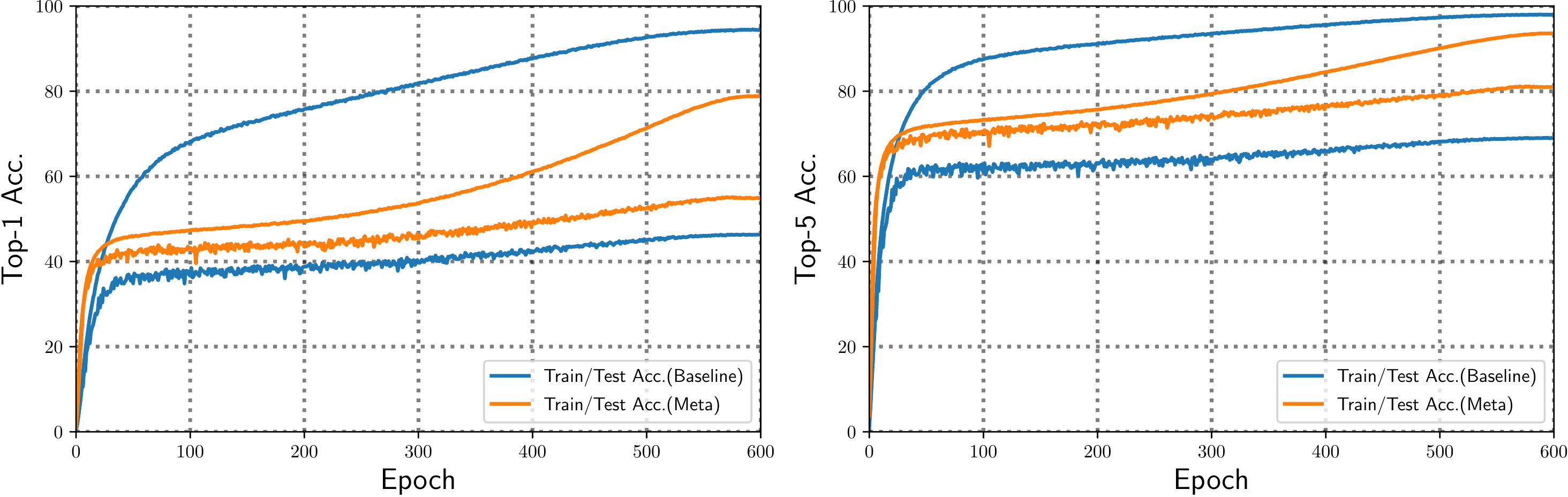}
    \caption{Accuracy curves of the baseline method and the proposed algorithm.
    }
    \label{fig:dynamics}
\end{figure*}

\begin{table*}[!t]
\begin{minipage}{0.47\textwidth}
% \begin{table*}[!t]
    \centering
    \caption{Ablation study of the meta-learning component and the mixup augmentation.
    The same number of labeled data is used as in \cref{tab:main-results}.}
    \vspace{3pt}
    \scalebox{0.77}{
    \begin{tabular}{c|cc|ccc}
        \Xhline{1pt}
        No. & Meta & Mix-Up & \textbf{SVHN} & \textbf{CIFAR-10} & \textbf{CIFAR-100} \\\hline
        1 & & & 9.76\% & 15.43\% & 38.74\% \\
        2 & \checkmark & & 3.68\% & 11.63\% & 35.40\% \\
        3 & & \checkmark & 5.60\% & 11.10\% & 32.67\% \\
        4 & \checkmark & \checkmark & \textbf{3.15\%} & \textbf{~~7.78\%} & \textbf{30.74\%} \\
        \Xhline{1pt}
    \end{tabular}
    }
    \label{tab:components-ablation}
% \end{table*}
\end{minipage}
\hfill
\begin{minipage}{0.5\textwidth}
\centering
\caption{
    Comparison with ADA-Net~\citep{wang2019semi} on both with and without mix-up setting. The number of provided labels are 4k and 1k in CIFAR-10 and SVHN, respectively.
}
\label{tab:compare-adanet}
\vspace{3pt}
\scalebox{0.83}{
\begin{tabular}{c|cc|cc}
    \Xhline{1pt}
    \multirow{2}[0]{*}{Dataset} & \multicolumn{2}{c}{\textbf{CIFAR-10}} & \multicolumn{2}{c}{\textbf{SVHN}} \\
    \cline{2-5}
     & w/o mixup & w mixup & w/o mixup & w/ mixup \\
    \hline
    ADA-Net & 18.67\% & 8.87\% & 10.76\% & 5.90\% \\
    \hline
    Ours & {\bf 11.63\%} & {\bf 7.78\%}  &  {\bf 3.86\%} & {\bf 3.15\%} \\
    \Xhline{1pt}
\end{tabular}
}
\end{minipage}
\end{table*}

\subsection{Results on the ImageNet}

The evaluation results with the ResNet-18 \citep{he2016deep} backbone
on the ImageNet dataset \citep{ILSVRC15} are summarized in \cref{tab:imagenet}.
The proposed algorithm performs well against the ADA-Net \citep{wang2019semi} in terms of top-5 accuracy.
%\xth{It is better to add more insight on why the proposed method could gain better performance than others.}
%Xinyu: because we take the labeled information into account? This is still repeating...
%
In addition, we demonstrate the accuracy curves of the baseline setting and our approach in \cref{fig:dynamics}, where the baseline setting means only 10\% of the total training examples are used as the labeled data used during training.
% \xth{I am confused about the baseline. What is the setting of baseline? Only use labeled data? It is better to give a clear introduction to the experimental setting.}
%Xinyu: explained
%
\cref{fig:dynamics} shows that merely involving 10\% training examples will lead to severe overfitting, as the training accuracy is very high while testing accuracy is pretty low. The problem is alleviated in our approach thanks to the explicit learning-to-generalize training scheme. Though the training accuracy is not higher than the baseline model, the proposed algorithm achieves better testing accuracy.
% \xth{Is involving 10\% training examples the setting of the baseline? If so, what is the meaning of labeled-only baseline? How does our approach alleviate the problem? Extend the sentence for details.}
%Xinyu: labeled-only means only 10\% of total data are used as labeled data, what's wrong with this?
%Xinyu: because our algorithm is explicitly learning to generalize? (again repeating?)
% \xth{写的比较表面，能不能更加深入一些？}
%
These results suggest that the consistency loss can effectively regularize the training and benefits 
to the generalization ability of the learned model.

\subsection{Ablation Studies}

\paragraph{Effectiveness of Components.}
We analyze the contributions of the meta-learning and
mixup augmentation components of the proposed algorithm. 
The experimental settings are the same as those in \cref{tab:main-results}.
\cref{tab:components-ablation} shows both components can significantly
improve the classification accuracy in the semi-supervised settings.
The mixup augmentation is a simple trick to solve the sample bias issue in SSL. Also, we adapt the mixup formulation (\ie mixing the actual labels and updated pseudo labels in Eq.~\eqref{eqn:mix-up}) to make it compatible with the meta learning framework. Such adaptation is non-trivial as indicated by the performance improvement on top of the mixup-only setting (see the last two rows of Table~\ref{tab:components-ablation}). 
Moreover, the meta-learning component can further improve performance with the presence of the mix-up augmentation,
indicating that meta-learning is orthogonal to the existing data augmentation techniques as a research direction.
To further verify the effectiveness of the proposed meta-learning component, we compare our method with ADA-Net on both with and without mixup setting as shown in Table~\ref{tab:compare-adanet}. We can conclude from the table that the proposed meta learning still outperforms than ADA-Net without using mixup augmentation.

\begin{figure*}[!tb]
\centering
    \begin{minipage}{0.37\textwidth}
        \centering
        \includegraphics[width=\linewidth]{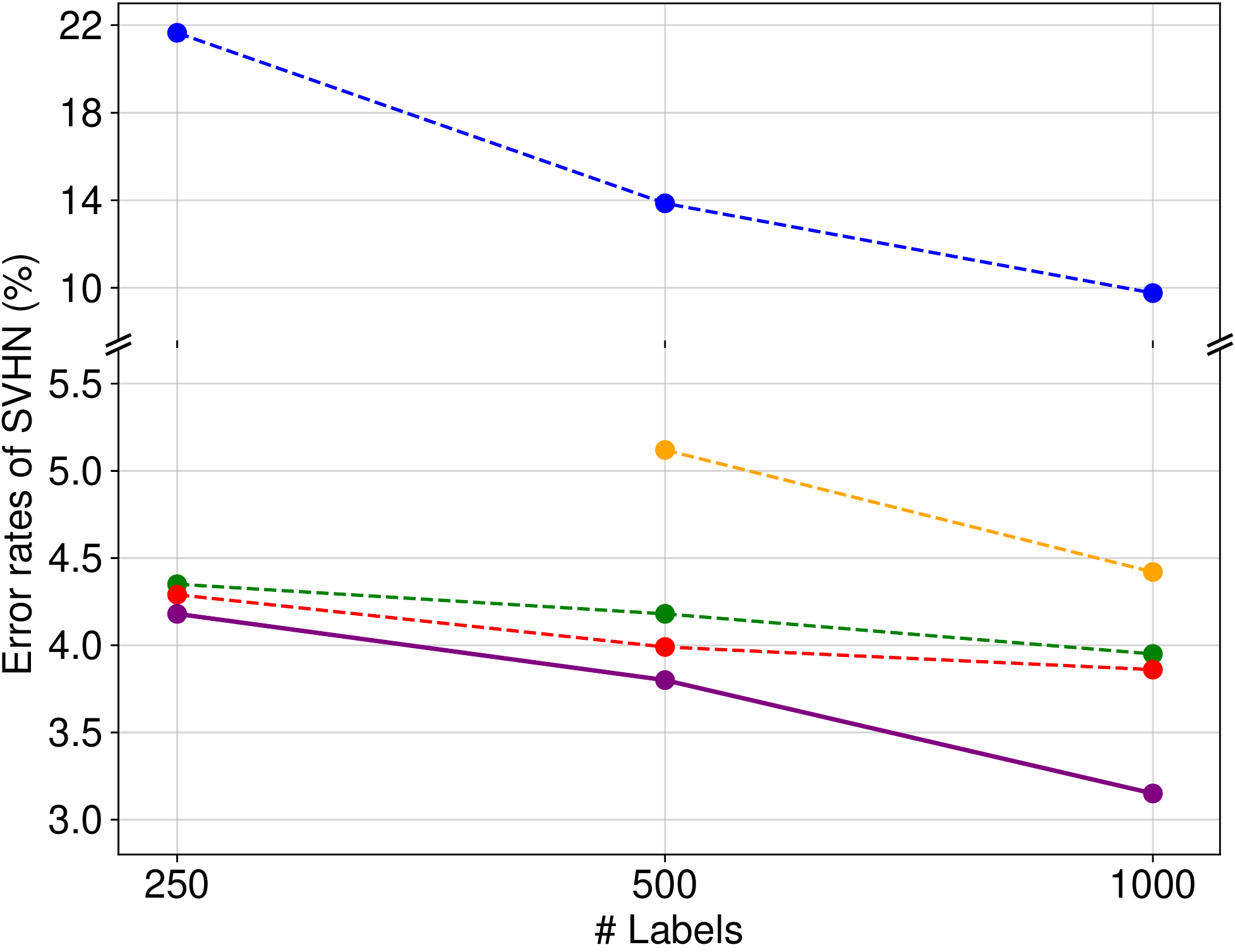}
    \end{minipage}%
    \hskip 5ex
    \begin{minipage}{0.37\textwidth}
        \centering
        \includegraphics[width=\linewidth]{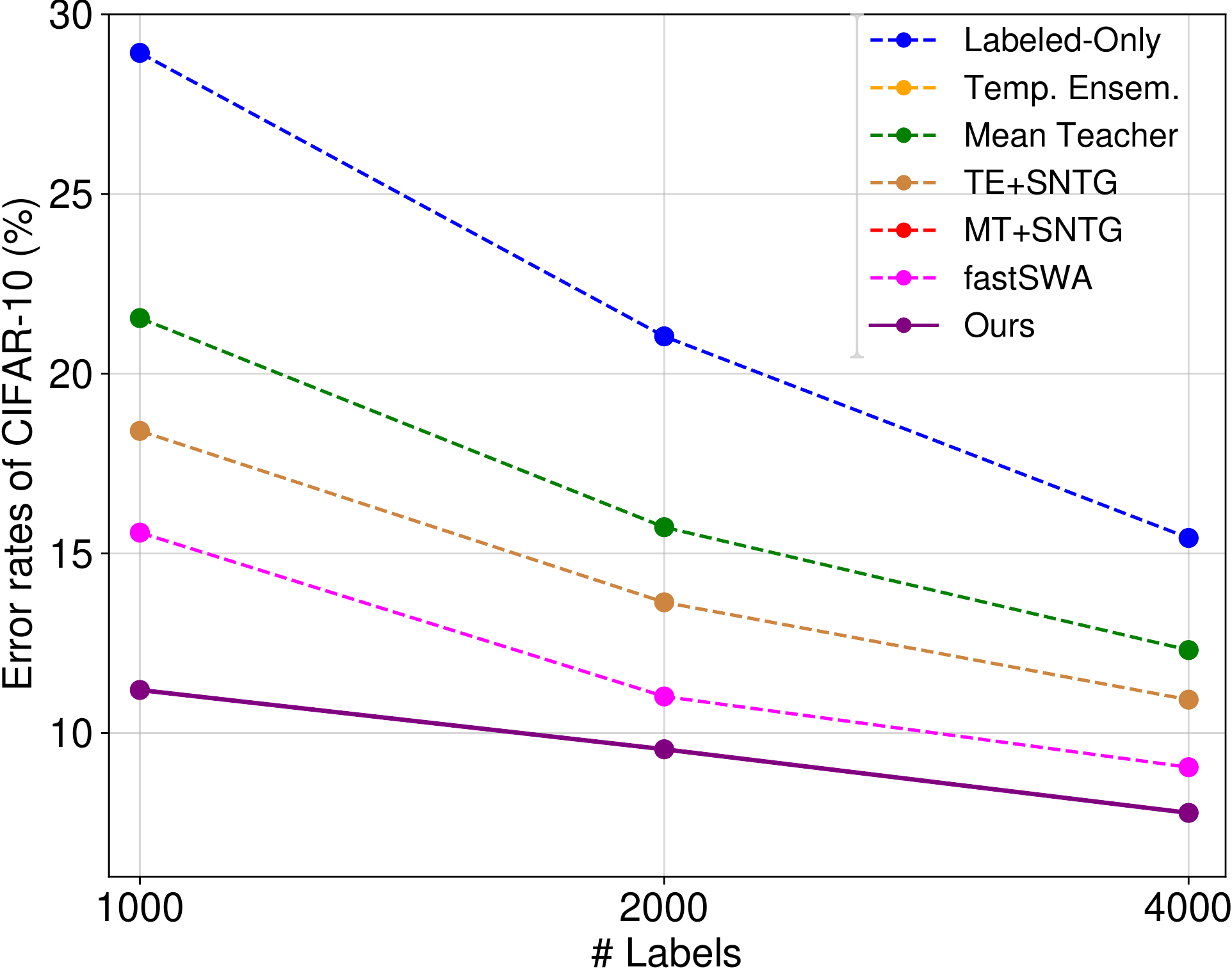}
    \end{minipage}%
\caption{Error rate \vs number of labeled examples.
}
\label{fig:empirical}
\end{figure*}

\begin{figure*}[!t]
    \centering
    \vspace{12pt}
    \subfigure{%
        \begin{minipage}{\textwidth}
            \centering
            \begin{overpic}[width=0.73\linewidth]{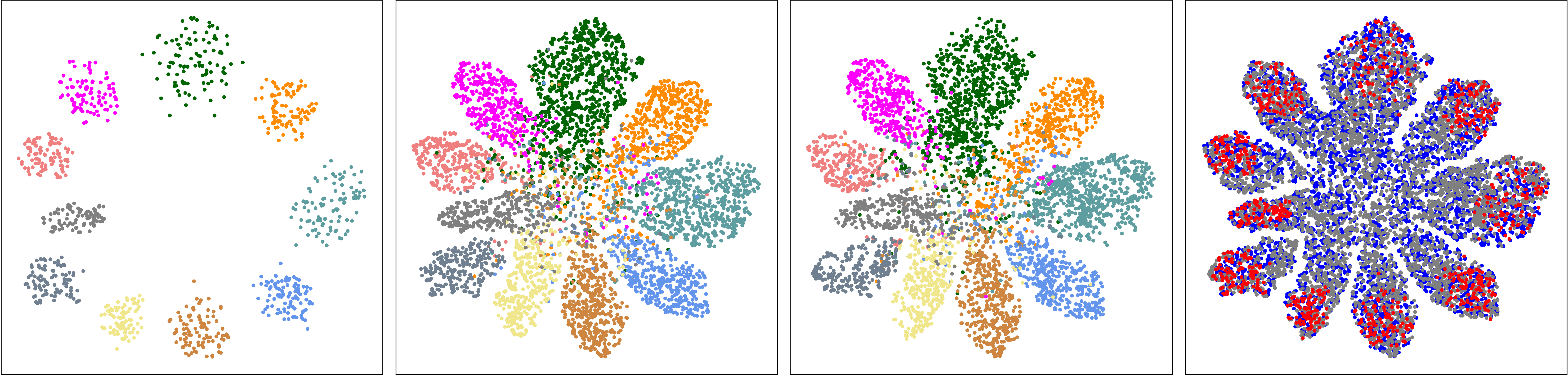}
                \put(-11,11){Baseline}
                \put(-7,-13){Ours}
                \put(8, 24.5){Labeled}
                \put(31.5, 24.5){Unlabeled}
                \put(60.3, 24.5){Test}
                \put(85.5, 24.5){All}
                \put(101, 14){\textcolor{red}{labeled}}
                \put(101, 11){\textcolor{blue}{unlabeled}}
                \put(101, 8){\textcolor{gray}{test}}
            \end{overpic}
        \end{minipage}%
    }
    % \vspace{-8pt}

    \subfigure{%
        \begin{minipage}{\textwidth}
            \centering
            \includegraphics[width=0.73\textwidth]{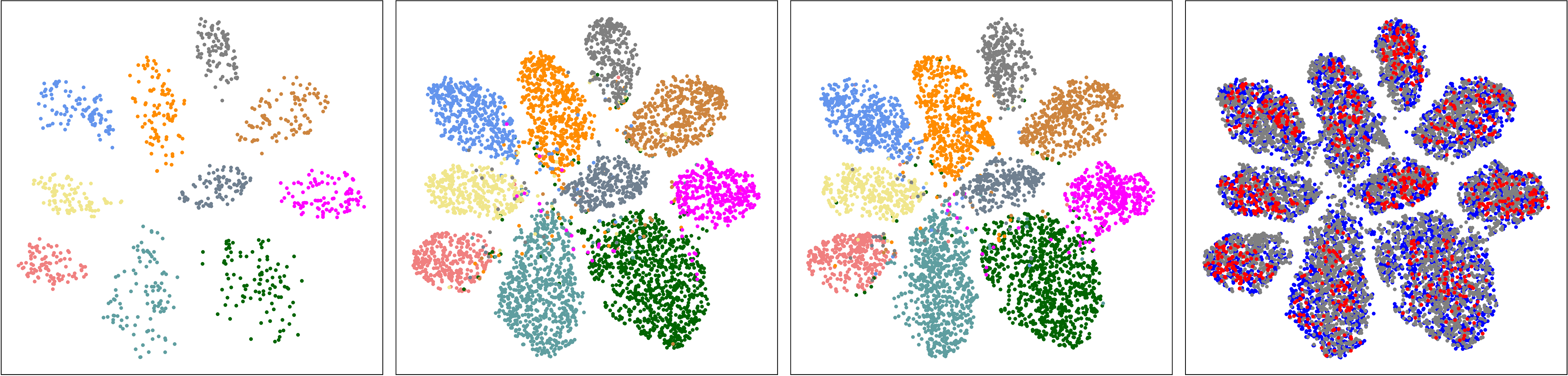}
        \end{minipage}%
    }
    \caption{
        Visualization of the SVHN features by the labeled-only baseline
        and our method.
        We extract features of the labeled (first column), unlabeled (second column),
        and test samples (third column) and project them to the two-dimensional space
        using t-SNE.
        In the first three columns, different categories are represented in different colors.
        In the fourth column, we plot the projected points all together to demonstrate
        the empirical distribution discrepancy.
        The labeled, unlabeled, and test examples are represented in
        \textcolor{red}{red}, \textcolor{blue}{blue}, and \textcolor{gray}{gray},
        respectively.
    }
    \vspace{-4pt}
    \label{fig:visualization}
\end{figure*}

\paragraph{Impact of \#Labels.}
As shown in~\cref{fig:empirical}, we evaluate the robustness of our method against the variation of the number of labels on the SVHN and CIFAR-10 datasets using the same experimental settings, except the number of labels as in~\cref{tab:main-results}.
%
% When the labels are relatively scarce, the accuracy of the labeled-only baseline degrades significantly.
We can find that the accuracy of the labeled-only baseline degrades heavily when reducing the number of labels.
In general, the scarcity of labeled data may lead to a worse generalization ability and more severe overfitting.
%
%MH: Be precise. This statement is not correct.
%In contrast, the meta-learning algorithm is robust against the scarcity of labeled examples.
However, the proposed SSL method can retain a relatively high performance under each setting thanks to our learning-to-generalize regularization.
% performs consistently better than other consistency-based approaches.
%
Therefore, the proposed method can effectively improve the generalization ability even with fewer labeled examples.

\paragraph{Feature Visualization.}

To further analyse the efficacy of our method, we visualize the SVHN features by projecting 128-dimensional features onto a two-dimensional space using the t-SNE \citep{vanDerMaaten2008} technique.
% \xth{Could give a citation to t-SNE paper.}
%Xinyu: added
%
For comparison, we also present the feature visualization of the labeled-only baseline method. 
As displayed in~\cref{fig:visualization}, there is a considerable empirical distribution discrepancy
between labeled and unlabeled examples of the baseline method.
In contrast, such discrepancy is reduced and the feature distributions of
different domains are aligned to some extent by our approach,
indicating the generalization ability is improved.
% \xth{We can add more details here. We can refer to the equations before and point out what connects the labeled and unlabeled data.}
%Xinyu: seems a lot of equations are involved...
%
Furthermore, considering the unlabeled and test samples,
we observe the margin among features of different categories is obviously clearer under our approach, which results in more discriminative features and better classification accuracy.

% \vspace{-6pt}
\section{Conclusion}
% \vspace{-4pt}

In this work, we address the semi-supervised learning problem in a meta-learning fashion. Specifically, we propose a new regularization term that takes the  model generalization ability into consideration in semi-supervised learning, which is different from most existing regularization terms that are based on the empirical assumption of the inner structure of the data manifold. Theoretically, we prove that the proposed algorithm is guaranteed to converge to a critical point at a convergence rate of $O(1/\epsilon^2)$. Extensive experimental results demonstrate that the proposed algorithm performs favorably against the state-of-the-art semi-supervised learning methods on the SVHN, CIFAR, and ImageNet datasets with better generalization ability.

\clearpage

% \bibliographystyle{abbrvnat}
% \setcitestyle{authoryear,open={((},close={))}}
\bibliography{semisup.bib}

\clearpage
\onecolumn
\appendix

\vspace{-15pt}
\section{Table of Notations}
\vspace{-3pt}

The notations in this work are summarized in~\cref{tab:notations}.

\begin{table}[H]
  \definecolor{mygray}{RGB}{220,220,220}
  \renewcommand{\arraystretch}{1}
  \centering
  \caption{Table of notations in this work.}
%   \vspace{-2pt}
  \resizebox{0.95\textwidth}{!}{
  \begin{tabular}{cl}
      \Xhline{1pt}
      \textbf{Symbol} & \multicolumn{1}{c}{\textbf{Description}} \\
      \hline\hline
      \multicolumn{2}{c}{\textbf{Data}}\\
      $\vx_{k}^{l}$ & The $k^{th}$ labeled training example \\
      \rowcolor{mygray} $\vx_{i}^{u}$ & The $i^{th}$ unlabeled training example \\
      \hline\hline
      \multicolumn{2}{c}{\textbf{Labels}}\\
      \rowcolor{mygray} $\vy_{k}$ & Actual label of $\vx_{k}^{l}$ \\
      $\widetilde{\vy}_{i}$ & Initialized proximal label of $\vx_{i}^{u}$ \\
      \rowcolor{mygray} $\widehat{\vy}_{i}$ & Updated proximal label of $\vx_{i}^{u}$ \\
      $\widetilde{\evy}_{i,j}$ & The $j^{th}$ entry of $\widetilde{\vy}_{i}$ (Same for $\widehat{\evy}_{i,j}$) \\
      \rowcolor{mygray} $\widetilde{\vy}$ & Proximal labels of the unlabeled mini-batch
      $\widetilde{\vy} = \left\{ \widetilde{\vy}_{i} : i = 1, \cdots, B^{u} \right\}$ (Seen as a vector) \\
      $\widetilde{\vy}_{t}$ & Proximal labels of the unlabeled mini-batch at the $t^{th}$ step\footnotemark \\
      \rowcolor{mygray} $\widetilde{\vy}_{i,t}$ & The $i^{th}$ proximal label of $\widetilde{\vy}_{t}$ \\
      \hline\hline
      \multicolumn{2}{c}{\textbf{Functions}}\\
      \rowcolor{mygray} $f(\cdot; \vtheta)$ & Convolutional Neural Network parameterized by $\vtheta$ \\
      $\Phi(\cdot, \cdot)$ & Non-negative function that measures discrepancy of distributions \\
      \rowcolor{mygray} $\mathcal{L}(\cdot, \cdot; \vtheta)$ & Loss function of the data-label pair when the model parameter is $\vtheta$ \\
      $G(\vtheta; \mathcal{D})$ & Validation loss on the dataset $\mathcal{D}$ when the model parameter is $\vtheta$ \\
      \rowcolor{mygray} $\nabla_{\vx}g, J_{\vx}g$ & The gradient or Jacobian function of a generic function $g$ \wrt $\vx$ \\
      \hline\hline
      \multicolumn{2}{c}{\textbf{Parameters}}\\
      \rowcolor{mygray} $\vtheta_{t}$ & Model parameters at the $t^{th}$ step \\
      $\widetilde{\vtheta}_{t}$ & Model parameters after the pseudo-update at the $t^{th}$ step (\ie Eq.~(4) in the main text) \\
      \rowcolor{mygray} $\theta_{t, l}$ & The $l^{th}$ entry of $\vtheta_{t}$ (Same for $\widetilde{\theta}_{t,l}$) \\
      \hline\hline
      \multicolumn{2}{c}{\textbf{Gradients}}\\
      \rowcolor{mygray} $\nabla \vtheta_{t}$ & Gradients of the pseudo-update at the $t^{th}$ step \\
      $\nabla \widehat{\vtheta}_{t}$ & Gradients of the actual update at the $t^{th}$ step \\
      \rowcolor{mygray} $\nabla \vtheta_{t}^{l}$ & Gradients of labeled mini-batch at the $t^{th}$ step \\
      $\nabla \widetilde{\vy}_{i}$ & Gradients of the proximal label $\widetilde{\vy}_{i}$ \\
      \rowcolor{mygray} $\nabla \widetilde{\vy}$ & Gradients of proximal labels
      $\nabla \widetilde{\vy} = \left\{ \nabla \widetilde{\vy}_{i} : i = 1, \cdots, B^{u} \right\}$ (Seen as a vector) \\
      \hline\hline
      \multicolumn{2}{c}{\textbf{Configurations}}\\
      \rowcolor{mygray} $\alpha_{t},~\beta_{t}$ & Regular learning rate and meta learning rate at the $t^{th}$ step \\
      $N^{l},~N^{u}$ & Numbers of labeled examples and unlabeled examples \\
      \rowcolor{mygray} $B^{l},~B^{u}$ & batch sizes for labeled examples and unlabeled examples \\
      \Xhline{1pt}
  \end{tabular}
  }
  \label{tab:notations}
\end{table}
\footnotetext{
  For a bit abuse of notations, the subscript $t$ or $\tau$ of $\widetilde{\vy}$
  specify the current step number,
  while subscript $(i,j)$ of indicates the $j^{th}$ entry of the $i^{th}$ proximal label.
  The step subscript is ommited when there is no ambiguity.
  }
\clearpage

\section{Convergence Analysis of Semi-Supervised Learning with Meta-Gradient}

\subsection{Lemma of Lipschitz Continuity}

\begin{lemma}
    Let
    \begin{equation}
      G(\vtheta;\mathcal{D}^{l}) = \frac{1}{N^{l}} \sum_{k=1}^{N^{l}} \mathcal{L}(\vx_{k}^{l}, \vy_{k}; \vtheta_t)
    \end{equation}
    be the loss function of the labeled examples.
    Assume 
    \begin{enumerate}[label=(\roman*),itemindent=0pt]
        \item the gradient function $\nabla_{\vtheta}G$ is Lipschitz-continuous with a Lipschitz constant $L_0$;
        and
        \item the norm of the Jacobian matrix of $f$ \wrt $\vtheta$ is upper-bounded by a constant $M$, \ie
        \begin{equation}
          \left\| J_{\vtheta} f(\vx_{i}^{u}; \vtheta) \right\| \leq M,
          \quad
          \forall \, i \in \left\{1, \cdots, N^{u}\right\}. 
        \end{equation}
      \end{enumerate}
    If the labeled data loss is considered as a function of the pseudo-targets
    $\widetilde{\vy} = \{\widetilde{\vy}_{i}:i=1,\cdots,B^{u}\}$, \ie
    $H(\widetilde{\vy}) = G(\widetilde{\vtheta}_{t+1}(\widetilde{\vy}))$,
    then the gradient function $\nabla_{\widetilde{\vy}} H$ is also Lipschitz-continuous and its Lipschitz constant is
    upper-bounded by $4 \alpha_{t}^{2} M^{2} L_{0}$.
\label{lemma:lipschitz}
\end{lemma}

\begin{proof}
    Recall the SGD update formula
    \begin{equation}
        \widetilde{\vtheta}_{t+1} = \vtheta_{t} - \frac{\alpha_{t}}{B^{u}} \sum_{i=1}^{B^{u}}
        \nabla_{\vtheta} \mathcal{L}(\vx_{i}^{u}, \widetilde{\vy}_{i}; \vtheta_{t}),
    \end{equation}
    and we have
    \begin{equation}
        \frac{\partial \widetilde{\theta}_{t+1,l}}{\partial \widetilde{\evy}_{i,j}}
        = -\frac{\alpha_{t}}{B^{u}} \frac{\partial^{2} \mathcal{L}}{\partial
        \widetilde{\evy}_{i,j} \partial \theta_{l}}(\vx_{i}^{u}, \widetilde{\vy}_{i}; \vtheta_{t}).
    \end{equation}
    Then, we expand the partial derivative of each entry $\widetilde{\evy}_{i,j}$:
    \begin{equation}
        \begin{split}
            \frac{\partial H}{\partial \widetilde{\evy}_{i,j}} =& 
            \frac{1}{N^{l}} \sum_{k=1}^{N^{l}} \sum_{l}
            \frac{\partial \mathcal{L}}{\partial \evtheta_{l}}(\vx_{k}^{l}, \vy_{k}; \widetilde{\vtheta}_{t+1})
            \frac{\partial \widetilde{\evtheta}_{t+1,l}}{\partial \widetilde{\evy}_{i,j}} \\
            =& -\frac{\alpha_{t}}{B^{u}N^{l}} \sum_{k=1}^{N^{l}} \sum_{l}
            \frac{\partial \mathcal{L}}{\partial \evtheta_{l}}(\vx_{k}^{l}, \vy_{k}; \widetilde{\vtheta}_{t+1})
            \frac{\partial^{2} \mathcal{L}}{\partial \widetilde{\evy}_{i,j} \partial \evtheta_{l}}
            (\vx_{i}^{u}, \widetilde{\vy}_{i}; \vtheta_{t}) \\
            =& -\frac{\alpha_{t}}{B^{u}N^{l}} \sum_{k=1}^{N^{l}} \nabla_{\vtheta}^{\top} \mathcal{L}
            (\vx_{k}^{l}, \vy_{k}; \widetilde{\vtheta}_{t+1}) \cdot
            \nabla_{\vtheta} \frac{\partial \mathcal{L}}{\partial \widetilde{\evy}_{i,j}}
            (\vx_{i}^{u}, \widetilde{\vy}_{i}; \vtheta_{t}) \\
            =& -\frac{\alpha_{t}}{B^{u}} \nabla_{\vtheta}^{\top} G(\widetilde{\vtheta}_{t+1}) \cdot
            \nabla_{\vtheta} \frac{\partial \mathcal{L}}{\partial \widetilde{\evy}_{i,j}}
            (\vx_{i}^{u}, \widetilde{\vy}_{i}; \vtheta_{t}).
        \end{split}
        \label{eq:nabla-y}
    \end{equation}
    Then, for arbitrary $\widetilde{\vy}^{1}$ and $\widetilde{\vy}^{2}$,
    \begin{equation}
        \begin{split}
            &\left. \frac{\partial H}{\partial \widetilde{\evy}_{i,j}}
            \right|_{\widetilde{\vy}=\widetilde{\vy}^{1}} - 
            \left. \frac{\partial H}{\partial \widetilde{\evy}_{i,j}}
            \right|_{\widetilde{\vy}=\widetilde{\vy}^{2}} \\
            =& \frac{\alpha_{t}}{B^{u}} \left(
            \nabla_{\vtheta}^{\top} G(\widetilde{\vtheta}_{t+1}^{2})
            \cdot \nabla_{\vtheta}
            \frac{\partial \mathcal{L}}{\partial \widetilde{\evy}_{i,j}}
            (\vx_{i}^{u}, \widetilde{\vy}_{i}^{2}; \vtheta_{t})
            -\nabla_{\vtheta}^{\top} G(\widetilde{\vtheta}_{t+1}^{1})
            \cdot \nabla_{\vtheta} \frac{\partial \mathcal{L}}{\partial \widetilde{\evy}_{i,j}}
            (\vx_{i}^{u}, \widetilde{\vy}_{i}^{1}; \vtheta_{t}) \right) \\
            =& \frac{\alpha_{t}}{B^{u}} \left( \nabla_{\vtheta}^{\top}
            \frac{\partial \mathcal{L}}{\partial \widetilde{\evy}_{i,j}}
            (\vx_{i}^{u}, \widetilde{\vy}_{i}^{2}; \vtheta_{t}) \cdot
            \left( \nabla_{\vtheta} G(\widetilde{\vtheta}_{t+1}^{2}) -
            \nabla_{\vtheta} G(\widetilde{\vtheta}_{t+1}^{1}) \right) \right. + \\
            &\hspace{40pt} \left. \nabla_{\vtheta}^{\top}G(\widetilde{\vtheta}_{t+1}^{1}) \cdot
            \left( \nabla_{\vtheta}
            \frac{\partial \mathcal{L}}{\partial \widetilde{\evy}_{i,j}}
            (\vx_{i}^{u}, \widetilde{\vy}_{i}^{2}; \vtheta_{t})
            - \nabla_{\vtheta} \frac{\partial \mathcal{L}}{\partial \widetilde{\evy}_{i,j}}
            (\vx_{i}^{u}, \widetilde{\vy}_{i}^{1}; \vtheta_{t}) \right)
            \right),
        \end{split}
    \end{equation}
    where $\widetilde{\vtheta}_{t+1}^{r} = \widetilde{\vtheta}_{t+1}(\widetilde{\vy}^{r}),~r=1,2$.
    As the MSE loss is used for unlabeled data, we have
    $\frac{\partial \mathcal{L}}{\partial \widetilde{\evy}_{i,j}}(\vx_{i}^{u}, \widetilde{\vy}_{i}; \vtheta_{t}) = -2(f_{j}(\vx_{i}^{u};\vtheta_{t}) - \widetilde{\evy}_{i,j})$.
    Here, $f_{j}$ denotes the $j^{th}$ entry of $f$.
    Therefore,
    \begin{equation}
        \begin{split}
            \left. \frac{\partial H}{\partial \widetilde{\evy}_{i,j}}
            \right|_{\widetilde{\vy} = \widetilde{\vy}^{1}} - 
            \left. \frac{\partial H}{\partial \widetilde{\evy}_{i,j}}
            \right|_{\widetilde{\vy}=\widetilde{\vy}^{2}} 
            &= -\frac{2\alpha_{t}}{B^{u}} \nabla_{\vtheta}^{\top}
            f_{j}(\vx_{i}^{u}; \vtheta_{t}) \cdot
            \left( \nabla_{\vtheta} G(\widetilde{\vtheta}_{t+1}^{2}) -
            \nabla_{\vtheta} G(\widetilde{\vtheta}_{t+1}^{1}) \right), \\
            \nabla_{\widetilde{\vy}_{i}} H(\widetilde{\vy}^{1}) -
            \nabla_{\widetilde{\vy}_{i}} H(\widetilde{\vy}^{2})
            &= -\frac{2\alpha_{t}}{B^{u}} J_{\vtheta}
            f(\vx_{i}^{u}; \vtheta_{t}) \cdot
            \left( \nabla_{\vtheta} G(\widetilde{\vtheta}_{t+1}^{2}) -
            \nabla_{\vtheta} G(\widetilde{\vtheta}_{t+1}^{1}) \right).
        \end{split}
    \end{equation}
    By taking the norm, we have
    \begin{equation}
        \left\|\nabla_{\widetilde{\vy}_{i}} H(\widetilde{\vy}^{1}) -
        \nabla_{\widetilde{\vy}_{i}} H(\widetilde{\vy}^{2}) \right\| \leq \frac{2\alpha_{t}}{B^{u}} \left\| J_{\vtheta}
        f(\vx_{i}^{u}; \vtheta_{t}) \right\|
        \left\| \nabla_{\vtheta} G(\widetilde{\vtheta}_{t+1}^{2}) -
        \nabla_{\vtheta} G(\widetilde{\vtheta}_{t+1}^{1}) \right\|.
    \end{equation}
    By assumptions, we have
    \begin{equation}
      \begin{split}
        \left\| J_{\vtheta}f(\vx_{i}^{u}; \vtheta_{t}) \right\| &\leq M,\\
        \left\| \nabla_{\vtheta} G(\widetilde{\vtheta}_{t+1}^{2}) -
        \nabla_{\vtheta} G(\widetilde{\vtheta}_{t+1}^{1}) \right\| &\leq
        L_{0} \left\| \widetilde{\vtheta}_{t+1}^{2} - \widetilde{\vtheta}_{t+1}^{1} \right\|.
      \end{split}
    \end{equation}
    Considering
    \begin{equation}
      \begin{split}
        \left\| \widetilde{\vtheta}_{t+1}^{2} - \widetilde{\vtheta}_{t+1}^{1} \right\|
        &= \frac{\alpha_{t}}{B^{u}} \left\| \sum_{i=1}^{B^{u}}
        \left(\nabla_{\vtheta}\mathcal{L}(\vx_{i}^{u},\widetilde{\vy}_{i}^{2};\vtheta_{t}) -
        \nabla_{\vtheta}\mathcal{L}(\vx_{i}^{u},\widetilde{\vy}_{i}^{1};\vtheta_{t})\right)\right\| \\
        &\leq \frac{\alpha_{t}}{B^{u}} \sum_{i=1}^{B^{u}} \left\|
        \nabla_{\vtheta}\mathcal{L}(\vx_{i}^{u},\widetilde{\vy}_{i}^{2};\vtheta_{t}) -
        \nabla_{\vtheta}\mathcal{L}(\vx_{i}^{u},\widetilde{\vy}_{i}^{1};\vtheta_{t})\right\| \\
        &= \frac{2\alpha_{t}}{B^{u}} \sum_{i=1}^{B^{u}} \left\| J_{\vtheta}f(\vx_{i}^{u};\vtheta_{t})
        \cdot (\widetilde{\vy}_{i}^{1} - \widetilde{\vy}_{i}^{2}) \right\| \\
        &\leq \frac{2\alpha_{t}}{B^{u}} \sum_{i=1}^{B^{u}} \left\| J_{\vtheta}f(\vx_{i}^{u};\vtheta_{t}) \right\|
        \left\|\widetilde{\vy}_{i}^{1} - \widetilde{\vy}_{i}^{2}\right\| \\
        &\leq 2 \alpha_{t} M \left\|\widetilde{\vy}^{1} - \widetilde{\vy}^{2}\right\|,
      \end{split}
    \end{equation}
    thus we have
    \begin{equation}
      \begin{split}
        \left\| \nabla_{\widetilde{\vy}_{i}} H(\widetilde{\vy}^{1}) -
        \nabla_{\widetilde{\vy}_{i}} H(\widetilde{\vy}^{2}) \right\|
        &\leq \frac{4 \alpha_{t}^{2} M^{2} L_{0}}{B^{u}} \left\|\widetilde{\vy}^{1}
        - \widetilde{\vy}^{2}\right\|, \\
        \left\| \nabla_{\widetilde{\vy}} H(\widetilde{\vy}^{1}) -
        \nabla_{\widetilde{\vy}} H(\widetilde{\vy}^{2}) \right\|
        &\leq \sum_{i=1}^{B^{u}} \left\| \nabla_{\widetilde{\vy}_{i}} H(\widetilde{\vy}^{1}) -
        \nabla_{\widetilde{\vy}_{i}} H(\widetilde{\vy}^{2}) \right\|
        \leq 4 \alpha_{t}^{2} M^{2} L_{0} \left\|\widetilde{\vy}^{1} - \widetilde{\vy}^{2} \right\|.
      \end{split}
      \
      \label{eqn:lipschitz_bound}
    \end{equation}
    Therefore, $\nabla_{\widetilde{\vy}} H$ is Lipschitz-continuous
    with a Lipschitz constant $L_{t} \leq 4 \alpha_{t}^{2} M^{2} L_{0}$.
\end{proof}

\subsection{Proof of \cref{thm:convergence}}

\begin{proof}
  According to the Lagrange Mean Value Theorem, there exists $\xi \in (0, 1)$, such that 
  \begin{equation}
    H(\widehat{\vy}) = H(\widetilde{\vy})
    + \nabla_{\widetilde{\vy}}^{\top} H(\widetilde{\vy} + \xi (\widehat{\vy}-\widetilde{\vy}))
    \cdot (\widehat{\vy}-\widetilde{\vy}).
    \label{eqn:proof_thm1}
  \end{equation}
  Recall the update formula of the pseudo-targets, \ie
  $\widehat{\vy} = \widetilde{\vy} - \beta_{t} \nabla \widetilde{\vy}$.
  Then, by the Lipschitz-continuity of $\nabla_{\widetilde{\vy}} H$, we have
  \begin{equation}
    \begin{split}
      H(\widehat{\vy}) &= H(\widetilde{\vy}) - \beta_{t}
      \nabla_{\widetilde{\vy}}^{\top} H(\widetilde{\vy} - \xi \beta_{t} \nabla \widetilde{\vy}) \cdot
      \nabla \widetilde{\vy} \\
      &= H(\widetilde{\vy})- \beta_{t} \nabla_{\widetilde{\vy}}^{\top} H(\widetilde{\vy}) \cdot
      \nabla \widetilde{\vy} - \beta_{t} \left(\nabla_{\widetilde{\vy}}^{\top} H(\widetilde{\vy}
      - \xi \beta_{t} \nabla \widetilde{\vy}) - \nabla_{\widetilde{\vy}}^{\top} H(\widetilde{\vy})\right)
      \cdot \nabla \widetilde{\vy} \\
      &\leq  H(\widetilde{\vy}) - \beta_{t}
      \nabla_{\widetilde{\vy}}^{\top} H(\widetilde{\vy}) \cdot \nabla \widetilde{\vy} +
      \beta_{t}^{2} L_t \|\nabla \widetilde{\vy}\|_{2}^{2}
      \hspace{87pt} (\mbox{By}~\eqref{eqn:lipschitz_bound}) \\
      &= H(\widetilde{\vy}) - (\beta_t-\beta_t^2 L_t) ||\nabla \widetilde{\vy}||^2
      \hspace{80.5pt} (\mbox{Since}~\nabla \widetilde{\vy} = \nabla_{\widetilde{\vy}} H(\widetilde{\vy})) \\
      &\leq H(\widetilde{\vy}).
      \hspace{195.3pt} (\mbox{Since}~\beta_{t} <  L_{t}^{-1})
    \end{split}
  \end{equation}
  Therefore, $G(\vtheta_{t+1}) = H(\widehat{\vy}) \leq H(\widetilde{\vy}) = G(\vtheta_{t})$.
 
  Moreover, as long as $\alpha_{t}^{2} \beta_{t} < (4 M^{2} L_{0})^{-1}$ is satisfied,
  the equality holds if and only if $\nabla \widetilde{\vy} = \bm{0}$.
\end{proof}

\subsection{Proof of \cref{thm:convergence-rate}}

\begin{proof}
  According to \eqref{eqn:proof_thm1} in the proof of \cref{thm:convergence},
  we have
  \begin{equation}
    G(\vtheta_{t+1}) \leq G(\vtheta_{t}) - (\beta_{t} - \beta_{t}^{2} L_{t})
    \left\| \nabla \widetilde{\vy}_{t} \right\|^{2}
    \leq G(\vtheta_{t}) - (\beta_{t} - 4 \alpha_{t}^{2}\beta_{t}^{2} M^{2} L_{0}) \left\|
    \nabla \widetilde{\vy}_{t} \right\|^{2}.
  \end{equation}
  Therefore,
  \begin{equation}
    G(\vtheta_{t}) - G(\vtheta_{t+1}) \geq (\beta_{t} - 4 \alpha_{t}^{2}\beta_{t}^{2} M^{2} L_{0})
    \left\| \nabla \widetilde{\vy}_{t} \right\|^{2} \geq D_{1} \left\| \nabla \widetilde{\vy}_{t} \right\|^{2}.
  \end{equation}
  By taking the expectation, we have
  \begin{equation}
    \E_{1 \sim t} \left[ G(\vtheta_{t}) \right] - \E_{1 \sim t} \left[ G(\vtheta_{t+1}) \right]
    \geq D_1 \E_{1 \sim t} \left[ \left\| \nabla \widetilde{\vy}_{t} \right\|^{2} \right].
    \label{equ:expectation}
  \end{equation}
  Here, $\E_{1 \sim t}$ indicates the expectation is taken over the selected
  mini-batches of the first $t$ steps.
  Next, we show $\E_{1 \sim t} \left[ G(\vtheta_{t}) \right] = \E_{1 \sim t-1} \left[ G(\vtheta_{t}) \right]$,
  which is intuitive as the value of $\vtheta_{t}$ only relies on the selected batches of the first $t-1$ steps.
  We rigorously prove it with conditional expectation:
  \begin{equation}\label{equ:conditional}
    \E_{1 \sim t} \left[ G(\vtheta_{t}) \right]
    = \E_{1 \sim t-1} \left[ \E_{t} \left[ G(\vtheta_{t}) | 1 \sim t-1 \right] \right]
    = \E_{1 \sim t-1} \left[ G(\vtheta_{t}) \right].
  \end{equation}
  Here, the first equality comes from the \textit{law of total expectation}, while
  the second comes from the fact that $G(\vtheta_{t})$ is deterministic
  given the selected batches of the first $t-1$ steps.
  Besides, when $t = 1$, \eqref{equ:expectation} is adapted to
  \begin{equation}
    G(\vtheta_{1}) - \E_{1} \left[ G(\vtheta_{2}) \right]
    \geq D_1 \E_{1} \left[ \left\| \nabla \widetilde{\vy}_{1} \right\|^{2} \right],
    \label{equ:expectation-t=1}
  \end{equation}
  where $G(\vtheta_{1})$ is the loss of the initialized model parameters
  so the expectation is omitted.
  Then, by taking a summation over the first $T$ steps, we have
  \begin{equation}
    D_1 \sum_{t=1}^{T} \E_{1 \sim t} \left[ \left\| \nabla \widetilde{\vy}_{t} \right\|^{2} \right]
    \leq G(\vtheta_{1}) - \E_{1 \sim T} \left[ G(\vtheta_{T+1}) \right]
    \leq G(\vtheta_{1}).
  \end{equation}
  Therefore, there exists $\tau \in \{1, \cdots, T\}$, s.t.
  \begin{equation}
    \E_{1 \sim \tau} \left[ \left\| \nabla \widetilde{\vy}_{\tau} \right\|^{2} \right] \leq \frac{G(\vtheta_{1})}{D_{1}T}.
    \label{equ:deltay}
  \end{equation}
  Then, we attempt to build a relationship between $\nabla \widetilde{\vy}_{\tau}$
  and $\nabla_{\vtheta}G(\vtheta_{\tau})$.
  Similar to \cref{eq:nabla-y}, we have
  \begin{equation}
    \nabla \widetilde{\vy}_{i,\tau}
    = -\frac{\alpha_{\tau}}{B^{u}} \nabla_{\widetilde{\vy}_{i},\vtheta}^{2} \mathcal{L}(\vx_{i}^{u}, \widetilde{\vy}_{i}; \vtheta_{\tau}) \cdot \nabla_{\vtheta} G(\vtheta_{\tau})
    = \frac{2\alpha_{\tau}}{B^{u}} J_{\vtheta}f(\vx_{i}^{u}; \vtheta_{\tau}) \cdot \nabla_{\vtheta} G(\vtheta_{\tau}).
  \end{equation}
  Therefore,
  \begin{equation}
      \left\| \nabla \widetilde{\vy}_{\tau} \right\|^{2}
      = \sum_{i=1}^{B^{u}} \nabla^{\top} \widetilde{\vy}_{i,\tau} \cdot
      \nabla \widetilde{\vy}_{i,\tau}
      = \frac{4\alpha_{\tau}^{2}}{(B^{u})^{2}} \nabla_{\vtheta}^{\top} G(\vtheta_{\tau})
      \cdot \left( \sum_{i=1}^{B^{u}} J^{\top}_{\vtheta}f(\vx_{i}^{u}; \vtheta_{\tau})
      \cdot J_{\vtheta}f(\vx_{i}^{u}; \vtheta_{\tau}) \right) \cdot
      \nabla_{\vtheta} G(\vtheta_{\tau}).
  \end{equation}
  Now consider the potential unlabeled batches $\{\mB_{k} : k = 1, \cdots, N^{l}\}$
  of the $\tau^{th}$ step.
  Since, $\mathcal{D}^{l} \subseteq \mathcal{D}^{u}$, we can assume
  $\vx_{k}^{l} \in \mB_{k},~k = 1, \cdots, N^{l}$ and these batches are sampled with non-zero probabilities $\{p_{k} : k = 1, \cdots, N^{l}\}$.
  Let $p = \min_{k} p_{k} > 0$, and we have
\begin{equation}
\begin{split}
  \E_{1 \sim \tau} \left[ \left\| \nabla \widetilde{\vy}_{\tau} \right\|^{2} \right] 
  &= \E_{1 \sim \tau-1} \left[ \E_{\tau} \left[ \left\| \nabla \widetilde{\vy}_{\tau} \right\|^{2} \right] \right] \\
  &= \E_{1 \sim \tau-1} \left[ \frac{4\alpha_{\tau}^{2}}{(B^{u})^{2}} \nabla_{\vtheta}^{\top} G(\vtheta_{\tau})
  \cdot \E_{\tau} \left[ \sum_{i=1}^{B^{u}} J^{\top}_{\vtheta}f(\vx_{i}^{u}; \vtheta_{\tau})
  \cdot J_{\vtheta}f(\vx_{i}^{u}; \vtheta_{\tau}) \right] \cdot \nabla_{\vtheta} G(\vtheta_{\tau}) \right] \\
  &\geq \E_{1 \sim \tau-1} \left[ \frac{4\alpha_{\tau}^{2}}{(B^{u})^{2}} \nabla_{\vtheta}^{\top} G(\vtheta_{\tau})
  \cdot \left( \sum_{k=1}^{N^{l}} p_{k} \, J^{\top}_{\vtheta}f(\vx_{k}^{l}; \vtheta_{\tau})
  \cdot J_{\vtheta}f(\vx_{k}^{l}; \vtheta_{\tau}) \right) \cdot \nabla_{\vtheta} G(\vtheta_{\tau}) \right] \\
  &\geq \frac{4pD_{2}^{2}}{(B^{u})^{2}} \E_{1 \sim \tau-1} \left[ \nabla_{\vtheta}^{\top} G(\vtheta_{\tau})
  \cdot \left( \sum_{k=1}^{N^{l}} J^{\top}_{\vtheta}f(\vx_{k}^{l}; \vtheta_{\tau})
  \cdot J_{\vtheta}f(\vx_{k}^{l}; \vtheta_{\tau}) \right) \cdot \nabla_{\vtheta} G(\vtheta_{\tau}) \right].
\end{split}
\label{equ:total-expectation}
\end{equation}

  Note that similar to \cref{equ:conditional}, the inner expectation is also conditioned
  on the selected batches of the first $\tau - 1$ steps, which is equivalent to that
  conditioned on $\vtheta_{t}$.
  
  By applying the chain rule, we have
  \begin{equation}
    \nabla_{\vtheta} G(\vtheta) = \frac{2}{N^{l}} \sum_{k=1}^{N^{l}} J^{\top}_{\vtheta}f(\vx_{k}^{l}; \vtheta) \cdot
    \left( f(\vx_{k}^{l}; \vtheta) - \vy_{k} \right).
  \end{equation}
  Since both $f(\vx_{k}^{l}; \vtheta)$ and $\vy_{k}$ are distributions on the category space,
  there exists a constant $R > 0$, s.t. $\left\| f(\vx_{k}^{l}; \vtheta) - \vy_{k} \right\| \leq R$.
  Therefore,
  \begin{equation}
    \begin{split}
      &\,\sum_{k=1}^{N^{l}} J^{\top}_{\vtheta}f(\vx_{k}^{l}; \vtheta_{\tau})
      \cdot J_{\vtheta}f(\vx_{k}^{l}; \vtheta_{\tau}) \\
      \succeq&\, \frac{1}{R^2} \sum_{k=1}^{N^{l}} J^{\top}_{\vtheta}f(\vx_{k}^{l};
      \vtheta_{\tau}) \cdot (f(\vx_{k}^{l}; \vtheta_{\tau}) - \vy_{k}) \cdot
      (f(\vx_{k}^{l}; \vtheta_{\tau}) - \vy_{k})^{\top} \cdot J_{\vtheta}f(\vx_{k}^{l}; \vtheta_{\tau}) \\
      \succeq&\, \frac{1}{N^{l} R^2} \left( \sum_{k=1}^{N^{l}} J^{\top}_{\vtheta}f(\vx_{k}^{l};
      \vtheta_{\tau}) \cdot (f(\vx_{k}^{l}; \vtheta_{\tau}) - \vy_{k}) \right) \cdot
      \left( \sum_{k=1}^{N^{l}} J^{\top}_{\vtheta}f(\vx_{k}^{l};
      \vtheta_{\tau}) \cdot (f(\vx_{k}^{l}; \vtheta_{\tau}) - \vy_{k}) \right)^{\top} \\
      =&\, \frac{N^{l}}{4R^2} \nabla_{\vtheta} G(\vtheta_{\tau}) \cdot \nabla_{\vtheta}^{\top} G(\vtheta_{\tau}).
    \end{split}
    \label{equ:cauchy}
  \end{equation}
  
  Here, the symbol $\succeq$ indicates certain matrix relationship
  where $\mA \succeq \mB$ means $\mA - \mB$ is a positive semidefinite matrix.
  
  We prove the first inequality in \eqref{equ:cauchy} with simplified notations.
  Suppose $\vv$ is a vector and $\mA$ is a matrix of proper dimension.
  Then, we show that if $\left\| \vv \right\| \leq R$, then
  $R^{2} \mA^{\top} \mA \succeq \mA^{\top} \vv \vv^{\top} \mA$.
  For an arbitrary vector $\vu$ of proper dimension, we have
  \begin{equation}
      \vu^{\top} \mA^{\top} \vv \vv^{\top} \mA \vu = \left\| \vv^{\top} \mA \vu \right\|^{2} 
      \leq \left\| \vv \right\|^{2} \left\| \mA \vu \right\|^{2}
      \leq R^{2} \left\| \mA \vu \right\|^{2}
      = R^{2} \vu^{\top} \mA^{\top} \mA \vu.
  \end{equation}
  By definition, $R^{2} \mA^{\top} \mA - \mA^{\top} \vv \vv^{\top} \mA$ is positive semidefinite.
  The second inequality in \eqref{equ:cauchy} comes from the Cauchy-Schwartz inequality that
  $\E \left[ \mA^{\top} \mA \right] \succeq \E \left[ \mA^{\top} \right] \E \left[ \mA \right]$
  for any random matrix $\mA$.
  
  With \eqref{equ:total-expectation} and \eqref{equ:cauchy}, it is easy to show that
  \begin{equation}\label{equ:cauchy2}
    \E_{1 \sim \tau} \left[ \left\| \nabla \widetilde{\vy}_{\tau} \right\|^{2} \right]
    \geq \frac{pD_{2}^{2} N^{l}}{(B^{u})^{2}R^{2}} \E_{1 \sim \tau-1}
    \left[ \left\| \nabla_{\vtheta} G(\vtheta_{\tau}) \right\|^{4} \right]
    \geq \frac{pD_{2}^{2} N^{l}}{(B^{u})^{2}R^{2}} \left( \E_{1 \sim \tau-1}
    \left[ \left\| \nabla_{\vtheta} G(\vtheta_{\tau}) \right\|^{2} \right] \right)^{2}.
  \end{equation}
  Again, the second inequality comes from the Cauchy-Schwartz inequality.
  Incorporating with \eqref{equ:deltay}, we have
  \begin{equation}
    \E_{1 \sim \tau-1}
    \left[ \left\| \nabla_{\vtheta} G(\vtheta_{\tau}) \right\|^{2} \right]
    \leq \frac{C}{\sqrt{T}},
    \quad
    \mbox{where}~C = \frac{B^{u} R}{D_2} \sqrt{\frac{G(\vtheta_{1})}{p N^l D_{1}}}.
  \end{equation}
  which concludes this proof.
\end{proof}

\section{Implementation Details}
Our implementation is based on the PyTorch \citep{steiner2019pytorch} library and
the proposed algorithm is evaluated on the SVHN \citep{netzer2011reading},
CIFAR \citep{krizhevsky2009learning}, and ImageNet \citep{ILSVRC15} datasets.

\paragraph{Evaluation on the SVHN and CIFAR datasets.}
As the standard evaluation protocol, 1k category-balanced labels are used
for supervision out of the 73,257 training examples of the SVHN dataset.
For the CIFAR-10 (\resp CIFAR-100) dataset, the number of labeled examples is 4k
(\resp 10k) out the 50k training examples.
For the backbone architectures, the Conv-Large architecture is the same as the one in previous work
\citep{Laine2017temporal,miyato2018virtual,tarvainen2017mean,athiwaratkun2018there,wang2019semi}.
The detailed configurations are summarized in \cref{tab:convlarge-arch}.
For the ResNet \citep{he2016deep} architecture,
we adopt the ResNet-26-2x96d Shake-Shake regularized architecture with 12 residual blocks
as in \citet{Gastaldi17ShakeShake}.
The same architecture is used in prior SSL methods \citep{tarvainen2017mean,athiwaratkun2018there}.
We follow a common practice of data augmentation, \ie
zero-padding of 4 pixels on each side of the image,
random crop of a $32 \times 32$ patch, and random horizontal flip,
for the CIFAR datasets, and omit the random horizontal flip for SVHN.
The meta learning rate $\beta_{t}$ is always set equal to
the regular learning rate $\alpha_{t}$.
We train from scratch for 400k iterations with an initial learning rate of 0.1,
and decay the learning rate by a factor of 10 at the end of 300k and 350k iterations.
We use the SGD optimizer with a momentum of 0.9,
and the weight decay is set to $10^{-4}$ for the CIFAR datasets,
and $5 \times 10^{-5}$ for SVHN.
The batch size is 128 for both labeled and unlabeled data.
The shape parameter $\gamma$ of the Beta distribution is set to 1.0 for the CIFAR
datasets, and 0.1 for SVHN, as suggested by \citet{wang2019semi}.

\paragraph{Evaluation on the ImageNet dataset.}
The large-scale ImageNet benchmark contains 1.28M training images of 1k fine-grained
classes.
We evaluate on the ResNet-18 \citep{he2016deep} backbone with 10\% labels.
The standard data augmentation strategy
\citep{vgg,he2016deep,xie2017aggregated} is adopted:
image resize such that the shortest edge is of 256 pixels,
random crop of a $224 \times 224$ patch, and random horizontal flip.
The overall batch size is 512, and the same optimizer as the aforementioned one
is employed with a weight decay of $10^{-4}$.
We train for 600 epochs in total, and decay the learning rate from 0.1
according to the cosine annealing strategy \citep{loshchilov-ICLR17SGDR}.
The shape parameter $\gamma$ is set to 1.0.

\begin{table*}[!h]
  \renewcommand{\arraystretch}{1.2}
  \centering
  \caption{
  Conv-Large \citep{tarvainen2017mean} Architecture.
  }
  % \resizebox{0.60\textwidth}{!}{
  \begin{tabular}{cccccc}
      \Xhline{1pt}
      \multirow{2}{*}{\textbf{Layer}} & \multicolumn{4}{c}{\textbf{Configurations}} &
      \multirow{2}{*}{\textbf{Output Size}} \\
      & \#Filters & Kernel Size & Stride & \#Paddings & \\ \hline
      Convolution & 128 & 3 & 1 & 1 & $32 \times 32$ \\
      Convolution & 128 & 3 & 1 & 1 & $32 \times 32$ \\
      Convolution & 128 & 3 & 1 & 1 & $32 \times 32$ \\
      MaxPooling & 128 & 2 & 2 & 0 & $16 \times 16$ \\
      Dropout & \multicolumn{4}{c}{$\text{Drop probability}=0.5$} & $16 \times 16$ \\
      Convolution & 256 & 3 & 1 & 1 & $16 \times 16$ \\
      Convolution & 256 & 3 & 1 & 1 & $16 \times 16$ \\
      Convolution & 256 & 3 & 1 & 1 & $16 \times 16$ \\
      MaxPooling & 128 & 2 & 2 & 0 & $8 \times 8$ \\
      Dropout & \multicolumn{4}{c}{$\text{Drop probability}=0.5$} & $8 \times 8$ \\
      Convolution & 512 & 3 & 1 & 0 & $6 \times 6$ \\
      Convolution & 256 & 1 & 1 & 0 & $6 \times 6$ \\
      Convolution & 128 & 1 & 1 & 0 & $6 \times 6$ \\
      AvgPooling & 128 & 6 & 1 & 0 & $1 \times 1$ \\
      Linear & \multicolumn{4}{c}{$128 \rightarrow 10$} & $1 \times 1$ \\
      \Xhline{1pt}
  \end{tabular}
  % }
  \label{tab:convlarge-arch}
\end{table*}

\end{document}